\documentclass{l4dc2025}
\usepackage{lipsum}
\usepackage{xspace}

\newcommand{\modelname}{\textsc{LiveNet}\xspace}
\usepackage{booktabs}
\usepackage{mathrsfs}
\usepackage{graphicx}
\usepackage{wrapfig}
\usepackage{enumitem}
\usepackage[font=small]{caption}
\usepackage{multirow}
\setlist[itemize,enumerate]{left=0.05cm}

\title[LiveNet]{\modelname: Robust, Minimally Invasive Multi-Robot Control for Safe and Live Navigation in Constrained Environments\\
{\small (Code and Videos at  \href{livenet-uva.github.io}{\textbf{livenet-uva.github.io}})}}
\usepackage{times}
\usepackage{subcaption}
\usepackage{graphicx}

\author{%
 \Name{Srikar Gouru} \Email{mcj2vb@virginia.edu}\\
 \addr University of Virginia
 \AND
 \Name{Siddharth Lakkoju} \Email{qmn4cj@virginia.edu}\\
 \addr University of Virginia%
 \AND
 \Name{Rohan Chandra} \Email{rohanchandra@virginia.edu}\\
 \addr University of Virginia
}

\begin{document}

\maketitle


\begin{abstract}
Robots in densely populated real-world environments frequently encounter constrained and cluttered situations such as passing through narrow doorways, hallways, and corridor intersections, where conflicts over limited space result in collisions or deadlocks among the robots. Current decentralized state-of-the-art optimization- and neural network-based approaches $(i)$ are predominantly designed for general open spaces, and $(ii)$ are overly conservative, either guaranteeing safety, or liveness, but not both. While some solutions rely on centralized conflict resolution, their highly invasive trajectories make them impractical for real-world deployment. This paper introduces \modelname, a fully decentralized and robust neural network controller that enables human-like yielding and passing, resulting in agile, non-conservative, deadlock-free, and safe, navigation in congested, conflict-prone spaces. \modelname is minimally invasive, without requiring inter-agent communication or cooperative behavior. The key insight behind \modelname is a unified CBF formulation for \textit{simultaneous} safety and liveness, which we integrate within a neural network for robustness. We evaluated \modelname in simulation and found that general multi-robot optimization- and learning-based navigation methods fail to even reach the goal, and while methods designed specially for such environments do succeed, they are $10$--$20\times$ slower, $4$--$5\times$ more invasive, and much less robust to variations in the scenario configuration such as changes in the start states and goal states, among others. We open-source the \modelname code at \href{https://github.com/srikarg89/LiveNet}{\textbf{https://github.com/srikarg89/LiveNet}}.

\end{abstract}


\section{Introduction}
\label{sec: introduction}

Large-scale multi-agent robot navigation has recently gained popularity for its applications in many fields, including warehouse robots, autonomous vehicles~(\cite{chandra2022gameplan, chandra2024towards}), unmanned aerial vehicles, and more (\cite{bogue2024role, rasheed2022review, raj2024rethinking, francis2023principles}). These systems frequently operate in constrained and cluttered environments, such as navigating doorways, intersections, or narrow hallways. Such scenarios often lead to conflicts, manifesting as deadlocks, collisions, or both, when multiple agents attempt to occupy the same limited space simultaneously (\cite{chandra2024deadlock}). In contrast, humans navigate these challenges effortlessly and intuitively, demonstrating agility and safety by dynamically modulating their speed and trajectory. This allows them to avoid collisions (ensuring safety) and prevent deadlocks (maintaining liveness—defined as the ability to continually make progress toward their goal) while ensuring smooth and efficient transitions. This work investigates a fundamental research question: \textit{How can robots emulate human-like agility, safety, and liveness in constrained environments?}
\begin{wrapfigure}{r}{0.6\linewidth}
    \centering
    \includegraphics[width=\linewidth]{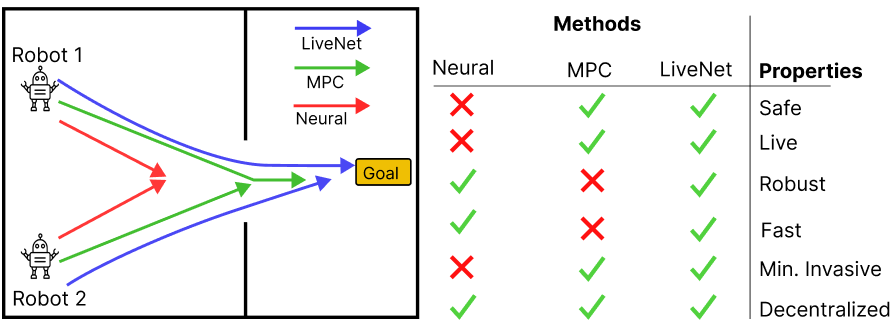}
    \caption{\modelname enables minimally invasive, robust, safe and deadlock-free navigation in constrained environments compared to existing methods.}
    \label{fig: hlm}
\end{wrapfigure}
We approach this question by designing a low-level controller capable of \textit{simultaneously} ensuring safety and liveness. Achieving only one of these properties is insufficient for replicating human-like behavior, such as agile yielding. Focusing solely on safety results in overly conservative behavior, while prioritizing liveness alone can lead to aggressive and potentially unsafe navigation. While some efforts have attempted to integrate both safety and liveness~(\cite{chandra2024deadlock, wang2017safety, zhou2017fast, chen2023multi, garg2024foundationmodelsrescuedeadlock, zinage2024decentralized}), these solutions often perturb the robots in an invasive manner that forces them to adopt suboptimal trajectories. For example,~(\cite{zhou2017fast, wang2017safety}) implement the right-hand-rule to induce clockwise movement in the event of a deadlock. Furthermore, in prioritizing safety, current methods only resolve deadlocks after they occur and all agents come to a stop. This is in stark contrast to how humans navigate by preemptively detecting and preventing the deadlock, without invasive maneuvers or delays. These issues highlight a critical gap in the literature towards achieving a robust, minimally invasive liveness without sacrificing safety.
An optimal algorithm to navigating in constrained and cluttered spaces, therefore, must satisfy the following criteria:

\begin{enumerate}[noitemsep]
    \item \textit{Decentralized / Non-cooperative}: The algorithm should not assume that agents can communicate with one another or with a centralized controller. The robots should only gain information about their surrounding environment through sensory perception.
    \item \textit{Robust}: Agents must be able to changes in the environment or in their own configurations.
    \item \textit{Safe}: The agents should maintain a predefined safety distance from other agents and obstacles.
    \item \textit{Live}: The agents should avoid deadlocks and continuously make progress towards the goal.
    \item \textit{Minimally invasive}: An optimal strategy should apply the least disruptive intervention using minimal perturbations to ensure conflict resolution.
    \item \textit{Dynamically feasible}: The agents should follow predefined non-holonomic and kinodynamic constraints, including bounds on the velocity and acceleration of the agent.
\end{enumerate}



\vspace{2.5mm}

\noindent\textbf{Main Contributions:} We propose a novel approach for optimal end-to-end learnable multi-robot navigation in constrained spaces such as doorways and corridor intersections. The key insight of our approach is to formulate both safety as well as liveness via differentiable CBFs within the neural network controller.

\begin{itemize}[noitemsep]

\item We propose the first safe, robust neural controller with provable liveness guarantees (Theorem~\eqref{thm: liveness}) for agile and smooth multi-robot navigation in constrained spaces.

\item Unlike prior methods, \modelname, while still fully decentralized, is minimally invasive, only perturbing the speed of a robot by the smallest amount necessary, without changing its direction, resulting in smoother and more optimal trajectories.

\item \modelname's control cycle frequency is between $10$--$20\times$ faster than MPC-based optimization approaches and $20\times$ faster than MACBF~(\cite{macbf2021}), a state of the art end-to-end learning-based multi-robot navigation baseline. 

\item \modelname is robust to changes in the environment and agent configurations. Given a wide range of diverse environments, \modelname succeeded in $30\%$ more scenarios compared to MPC-based baselines.


\end{itemize}

\section{Related Work}
\label{sec: related_work}

In this section we discuss the variety of methods that have been applied to safely navigate multi-robot deadlock avoidance scenarios.

\subsection{Simultaneous Safety and Deadlock Resolution Methods}

Most current methods rely on cooperative, predetermined behavior between the robots to resolve conflicts. For instance, in~(\cite{zhou2017fast, wang2017safety, zhu2022, chen2023multi}), the authors heuristically define a clockwise movement to establish the right of way (the rightmost agent moves first). Other deadlock resolution methods generate vehicle priorities through reservation systems like first come first serve (\cite{au2015autonomous}) or auctions with predefined bidding strategies (\cite{carlino2013auction, suriyarachchi2022gameopt}). Garg et al.~(\cite{garg2024foundationmodelsrescuedeadlock}) used large language models to act as a central arbiter to resolve the conflict. However, these heuristics result in larger perturbations than necessary to avoid the deadlock, and they often falter in scenarios with unpredictable agents where cooperation is not guaranteed. Chandra et al. (\cite{chandra2024deadlock}) model liveness as a control barrier function (CBF) within a receding-horizon control scheme, adjusting the agent's trajectory to avoid collisions and deadlocks in the immediate future. Although effective in certain settings, these methods lack robustness in new or highly dynamic scenarios, as they often require extensive hand-tuning for each unique situation.

\subsection{Other Multi-Robot Navigation Approaches}

Multi-agent path-finding (MAPF) algorithms such as conflict-based search (CBS) and its variants (\cite{sharon2015conflict}), $M^*$ (\cite{stern2019multi}), ICTS (\cite{sharon2013increasing}), and Uniform Cost Search (\cite{mao2024collision}, \cite{guo2024spatio}, \cite{mcnaughton2011motion}). Many MAPF techniques yield a globally optimal solution, but require centralized solvers~(\cite{chandra2023socialmapf}), discretized state spaces, and low-dimensional observation, state, and action spaces, thereby restricting its use primarily to simulation or offline as a coarse preprocessing step (\cite{ma2022graph}). These assumptions are prohibitive to most real world robots that possess non-holonomic and kinodynamic constraints.
Learning-based approaches use imitation learning~(\cite{hussein2018deep, yan2022mapless, daftry2017learning, macbf2021, barriernet2021}) and multi-agent reinforcement learning (MARL)~(\cite{liu2020pic, martinez2012multi, mehr2023maximum, chandra2024socialgym, wu2023intent}) to learn a navigation policy using supervised or unsupervised learning methods. While these methods offer robustness and scalability, they model safety as a learned behavior rather than a constraint. Thus, model performance is constrained to the training and testing datasets and may fail in novel, unobserved conditions. 
\begin{wraptable}{r}{0.6\textwidth}
    \centering
    \resizebox{.9\linewidth}{!}{
    \begin{tabular}{cc}
    \toprule
       Symbol  & Description  \\
       \midrule
       \multicolumn{2}{c}{\textit{Problem formulation (Section~\ref{sec: problem_formulation}})}\\
       \midrule
       $k$ & Number of agents\\
       $T$ & Planning horizon\\
       $\mathcal{X}$ & General continuous state space\\
       $\mathcal{X}_I$ & Set of initial states\\
       $\mathcal{X}_g$ & Set of final states\\
       $x^i_t$ & State of agent $i$ at time $t$\\
       $\Omega^i$ & Observation set of agent $i$\\
       $\mathcal{O}^i: \mathcal{X}\rightarrow \Omega^i$ & Agent $i$'s observation function\\
       $o^i_t$ & Observation of agent $i$ at time $t$\\
       $\Gamma^i$ & Agent $i$'s trajectory\\
       $\widetilde\Gamma^i$ & Agent $i$'s preferred or desired trajectory\\
       $\Psi^i$ & Agent $i$'s input control sequence\\
       $\mathcal{T}: \mathcal{X} \times \mathcal{U}^i \rightarrow \mathcal{X}$ & Environment transition dynamics (Equation~\ref{eq: discrete_dynamics})\\
       $\mathcal{U}^i$ & Action space for agent $i$\\
       $\mathcal{J}^i$ & Running cost for agent $i$ ($\mathcal{J}^i_t:\mathcal{X} \times  \mathcal{U}^i \rightarrow \mathbb{R}$)\\
       $\mathcal{J}^i_T$ & Terminal cost at time $T$\\
       $\mathcal{C}^i\left( x^i_t \right) \subseteq \mathcal{X}$ & Convex hull of agent $i$\\
       $\overline{\Gamma}^{i}$ & Agent $i$'s minimally invasive trajectory\\    
       $z_t^i \in \mathcal{X} \times \Omega^i$ & Model input, consisting of $x_t^i$ and $o_t^i$.\\
       $\mathcal{F}: \mathcal{X} \times \Omega^i \rightarrow \mathcal{U}^i$ & network defining agent $i$'s controller\\
       \midrule
       \multicolumn{2}{c}{\textit{Technical Approach (Section~\ref{sec: technical_approach}})}\\
       \midrule
       $b^i\left(z^i_t\right):\mathcal{X} \times \Omega^i \longrightarrow \mathbb{R}$ & Control barrier function (CBFs) \\
       $b_o^i, b_l^i\left(z_t^i\right)$ & Obstacle and liveness CBFs\\
       $L_f b^i\left(z^i_t\right),  L_g b^i\left(z^i_t\right)$ & Lie derivatives of $b^i\left(x^i_t\right)$ w.r.t $f$ and $g$.\\
       $p_t^i, \theta_t^i, v_t^i$ & Position, heading, and velocity of agent $i$\\
       $\widehat{\Gamma}^i, \widehat{\Psi}^i$ & State and input controls trajectory dataset\\
       $p^o, p^l$ & Penalty values defining the relaxation of the CBFs\\
         \bottomrule
    \end{tabular}
    }
    \caption{Summary of notation used in this paper.}
    \label{tab: notation}
\end{wraptable}

Optimization-based methods, particularly Model Predictive Control (MPC) with control barrier functions (CBFs), have been employed to calculate safe trajectories over short future horizons (\cite{mestres2024distributed, zhu2020trajectory, suriyarachchi2022gameopt, suriyarachchi2024gameopt+}). These receding-horizon control strategies iteratively solve an optimization problem at each step, adjusting the agent's trajectory to avoid collisions in the immediate future. Although effective in certain settings, these methods lack robustness in new or highly dynamic scenarios, as they often require extensive hand-tuning for each unique situation.
Another class of distributed optimization-based methods uses dynamic game theory to compute a Nash equilibria for similar problems that dictates all agents' trajectories (\cite{GameTheorySelfDriving, wang2021game,schwarting2021stochastic,sun2015game,sun2016stochastic,morimoto2003minimax,fridovich2020efficient,di2018differential, chandra2022game}). However, this requires knowledge of the other agents' objective functions, their desired trajectories, and their kinodynamic constraints (\cite{GameApproachMultiAgent2017}). 

The algorithms described above are able to either guarantee safety, liveness, or robustness to adapt to new scenarios well, but are unable to do all three. This research builds on foundational ideas of Neural Network based controllers and Control Barrier Functions (CBFs) (\cite{wang2017safety}) to provide safety and liveness while being robust.

\section{Problem Formulation}
\label{sec: problem_formulation}

In this section, we formulate the problem objective that we aim to solve. Notation for variables referenced is summarized in table \ref{tab: notation}.
We formulate the problem as the following partially observable stochastic game (POSG) (\cite{hansen2004dynamic}):
$\langle k, T, \mathcal{X}, \mathcal{U}^i, \mathcal{T}, \mathcal{J}^i, \mathcal{O}^i, \Omega^i \rangle$, where $k$ refers to the number of agents, and $T$ refers to the finite horizon length of the game. A superscript of $i$ refers to the $i^{th}$ agent, where $i \in [1, .., k]$ and a subscript of $t$ refers to discrete time step $t$ where $t \in [1, .., T]$. At any given time step $t$, agent $i$ has state $\mathbf{x}_t^i \in \mathcal{X}$ where $\mathcal{X}$ is the general, continuous state space. $\mathcal{U}^i$ is the continuous control space for robot $i$ representing the set of admissible inputs for $i$. Agent dynamics are defined by the transition function $\mathcal{T}: \mathcal{X} \times \mathcal{U}^i \rightarrow \mathcal{X}$ at time $t \in [1, .., T - 1]$. The cost function, $\mathcal{J}^i: \mathcal{X} \times \mathcal{U}^i \rightarrow \mathbb{R}$ is used to determine the cost of the specified control action in the agent's current state, and the terminal cost $\mathcal{J}^i_T: \mathcal{X} \rightarrow \mathbb{R}$ is used to calculate the cost of the terminal state $\mathbf{x}^i_N$. Each agent $i$ also has an observation $o_t^i \in \Omega^i$ which is determined via the observation function $o_t^i = \mathcal{O}^i (x_t^i)$. A discrete trajectory of agent $i$ is defined by $\Gamma^i = (\mathbf{x}_0^i, \mathbf{x}_1^i, ..., \mathbf{x}_T^i)$, and has a corresponding control input sequence $\Psi^i = (u_0^i, ..., u_{T-1}^i)$. Agents follow the control-affine dynamics $\mathbf{x}_{t+1}^i = f(\mathbf{x}_t^i) + g(\mathbf{x}_t^i, u_t^i)$, where $f,g$ are locally Lipschitz continuous functions. At any time $t$, each agent $i$ occupies a space given by $C^i(\mathbf{x}_t^i) \subseteq \mathcal{X}$. Two robots $i, j$ are considered colliding at time $t$ if $C^i(\mathbf{x}_t^i) \cap C^j(\mathbf{x}_t^j) \neq \emptyset$.

A \textit{Social Mini-Game} (SMG) is a variation of the generic POSG where each agent has a starting state, $\mathbf{x}_0^i \in \mathcal{X}_I$ and a goal state $\mathbf{x}_g^i \in \mathcal{X}_g$ where $\mathcal{X}_I$ and $\mathcal{X}_g$ are subsets of the continuous space $\mathcal{X}$ (\cite{chandra2024deadlock}). Additionally, each agent has a preferred trajectory, denoted by $\widetilde\Gamma^i$, which would be the desired trajectory that the agent would take in the absence of any other agents:
\vspace{-10pt}
{\small \begin{subequations}
\begin{align}
\left( \widetilde\Gamma^i, \widetilde\Psi^i\right) =& \arg\min_{(\Gamma^i,  \Psi^{i})} \sum_{t={0}}^{T-1} \mathcal{J}^i\left(\mathbf{x}^i_t, u^i_t\right) + \mathcal{J}^i_T\left(\mathbf{x}^i_T\right) \\
\text{s.t}\;\;  \mathbf{x}^i_{t+1}=& f\left(\mathbf{x}^i_t\right) + g\left(\mathbf{x}^i_t\right)u^i_t,\quad\forall t\in[1;T-1]\label{eq: discrete_dynamics}\\
u_{min} &\leq u_t^i \leq u_{max}\\
&\mathbf{x}^i_0\in \mathcal{X}_I\label{eq: libe_constraint},\quad \mathbf{x}^i_T\in \mathcal{X}_g
\end{align}
\label{eq: prob_1}
\end{subequations}}

\vspace{-10pt}
\begin{wrapfigure}{r}{0.5\linewidth}
    \centering
        \subfigure[{\footnotesize SMG Scenario}]{
        \label{fig: smg_example}
        \includegraphics[width=0.42\linewidth]{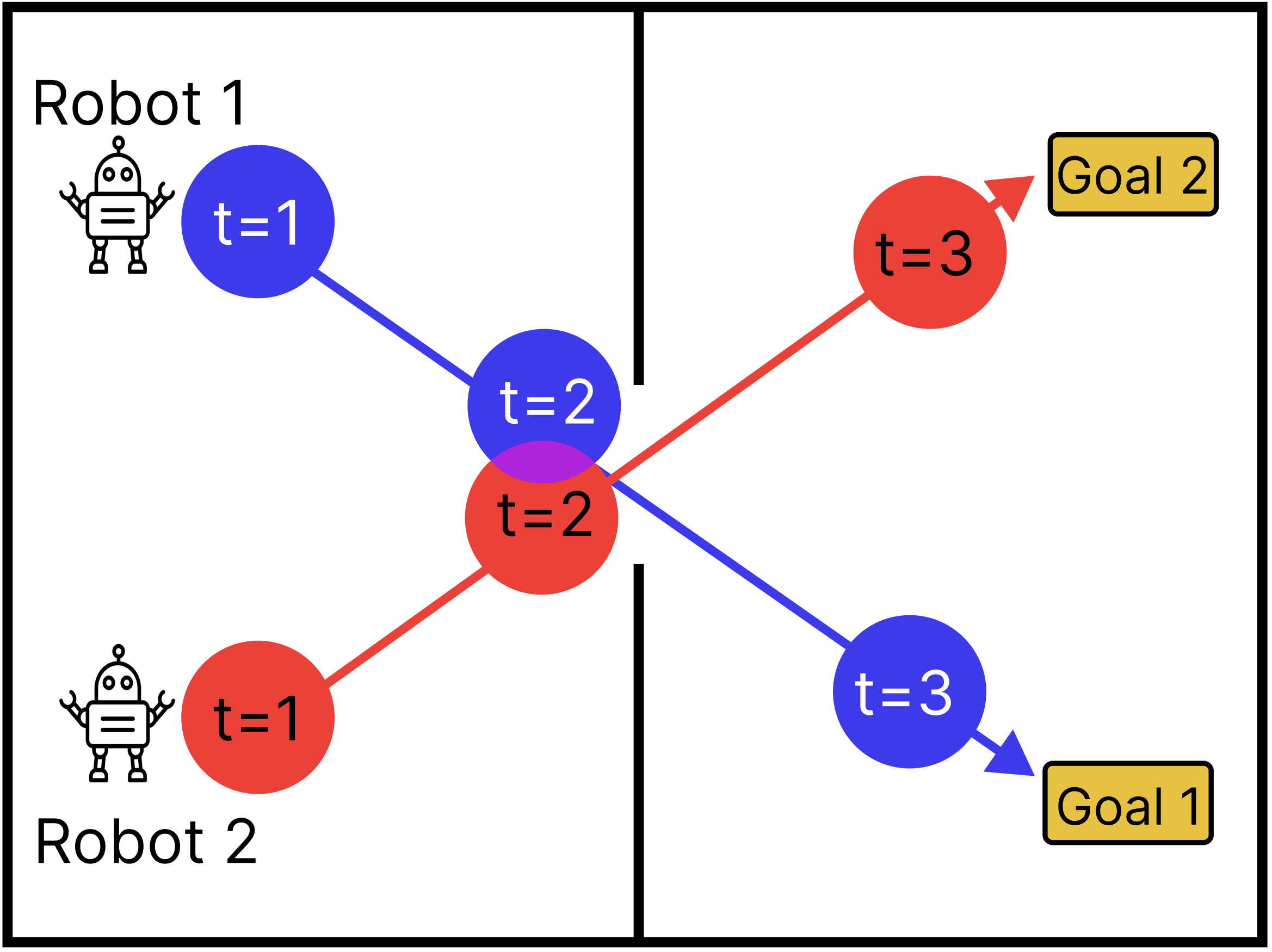}
    }\qquad
        \subfigure[{\footnotesize Non-SMG Scenario}]{
        \label{fig: non_smg_example}
        \includegraphics[width=0.42\linewidth]{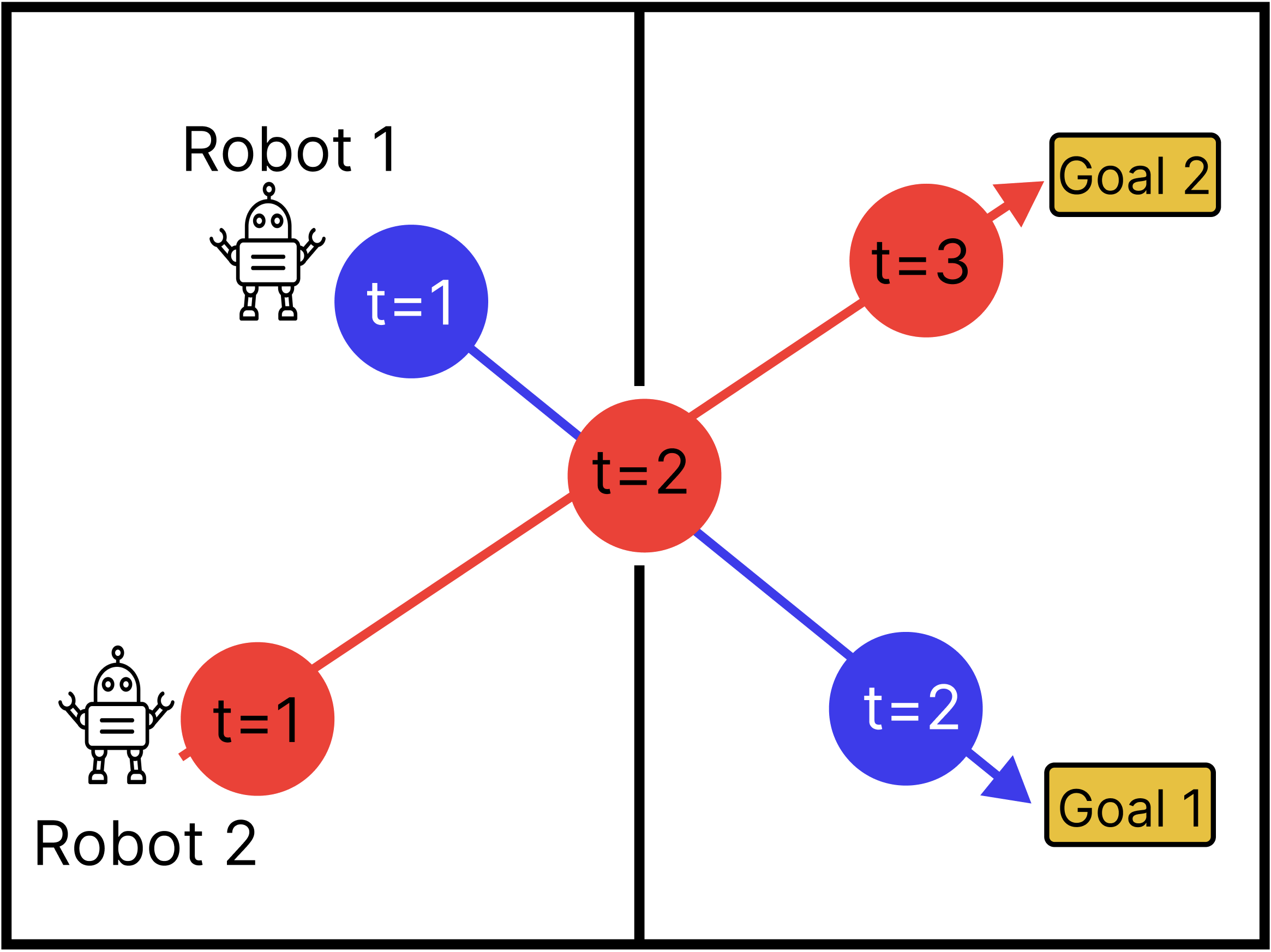}
    }
    \caption{Example SMG and Non-SMG scenarios with agent 1's desired trajectory in red and agent 2's desired trajectory in blue. Their starting and goal locations are indicated by $t = 1$ and $t = 4$, respectively, with $t$ being used to show the agents' time-parameterized desired trajectories. Collisions are shown in purple.}
    \label{fig: smg_formulation_scenarios}
\end{wrapfigure}

A game is considered a social mini-game when if for some $t \in [1, .., T]$, there exists at least one pair $i, j$ where $i \neq j$ such that $C^i(\mathbf{x}_t^i) \cap C^j(\mathbf{x}_t^j) \neq \emptyset$ where $\mathbf{x}_t^i \in \widetilde\Gamma^i$ and $\mathbf{x}_t^j \in \widetilde\Gamma^j$. As shown in Figure \ref{fig: smg_formulation_scenarios}, if any two agents' desired spatiotemporal trajectories intersect, even by a small amount, the game is considered an SMG, which results in either a collision or a deadlock~(\cite{chandra2024deadlock}). 
To prevent collisions and deadlocks in an SMG, agents $i$ and $j$ need to perturb their desired trajectories, $\Gamma^i$ and $\Gamma^j$, to avoid collisions. That is, we desire $\Gamma^i, \Gamma^j$ such that $C^i(\mathbf{x}_t^i) \cap C^j(\mathbf{x}_t^j) = \emptyset$ for all $t \in [1;T]$ where $\mathbf{x}_t^i \in \Gamma^i$ and $\mathbf{x}_t^j \in \Gamma^j$.

\noindent\textbf{Objective:} Our goal is to prevent both collisions and deadlocks in an SMG by perturbing the preferred trajectory, $\widetilde\Gamma^i$, in a \textit{minimally invasive} manner. We define the new trajectory, $\overline{\Gamma}^i$ to be minimally invasive if it perturbs the agent's velocity throughout the trajectory by the minimal amount possible without any spatial deviation from the preferred trajectory and avoids collisions with all other agents while following kinodynamic constraints. This mimics human yielding when crossing an intersection or passing through a doorway, where humans simply slow down to allow someone else to pass, but don't change their intended spatial path. Mathematically, this means that over the planning time-horizon $T$, the following constraints must be incorporated into Problem~\eqref{eq: prob_1}.

\vspace{-10pt}
{\small \begin{equation}
\sum_{t=0}^{T-1} D(\overline{p}_t^i, \widetilde\Gamma) \leq \epsilon^{(1)},\qquad \sum_{t=1}^{T-1} (\overline{v}_t^{i}  - \overline{v}_{t-1}^i) \leq \epsilon^{(2)}, \qquad C^i(\overline{\mathbf{x}}_t^i) \cap C^j(\mathbf{x}_t^j) = \emptyset 
\label{eq: added_constraints}
\end{equation}
}

\vspace{-10pt}

for all $i, j \in [0,k] \text{ s.t. } i \neq j$ where $D(\overline{p}_t^i, \widetilde\Gamma)$ represents the spatial deviation of point $\overline{p}^i$ from the desired trajectory $\widetilde\Gamma$. Formally, our goal is now to solve Problem~\eqref{eq: prob_1} with added constraints give by~\eqref{eq: added_constraints}. We define a minimally invasive solution as one that minimizes the value of $\epsilon^{(1)}$, and the event of a tie
minimizes $\epsilon^{(2)}$.

\section{\modelname: Technical Approach}
\label{sec: technical_approach}



\begin{wrapfigure}{r}{0.65\linewidth}
\centering
\includegraphics[width=\linewidth]{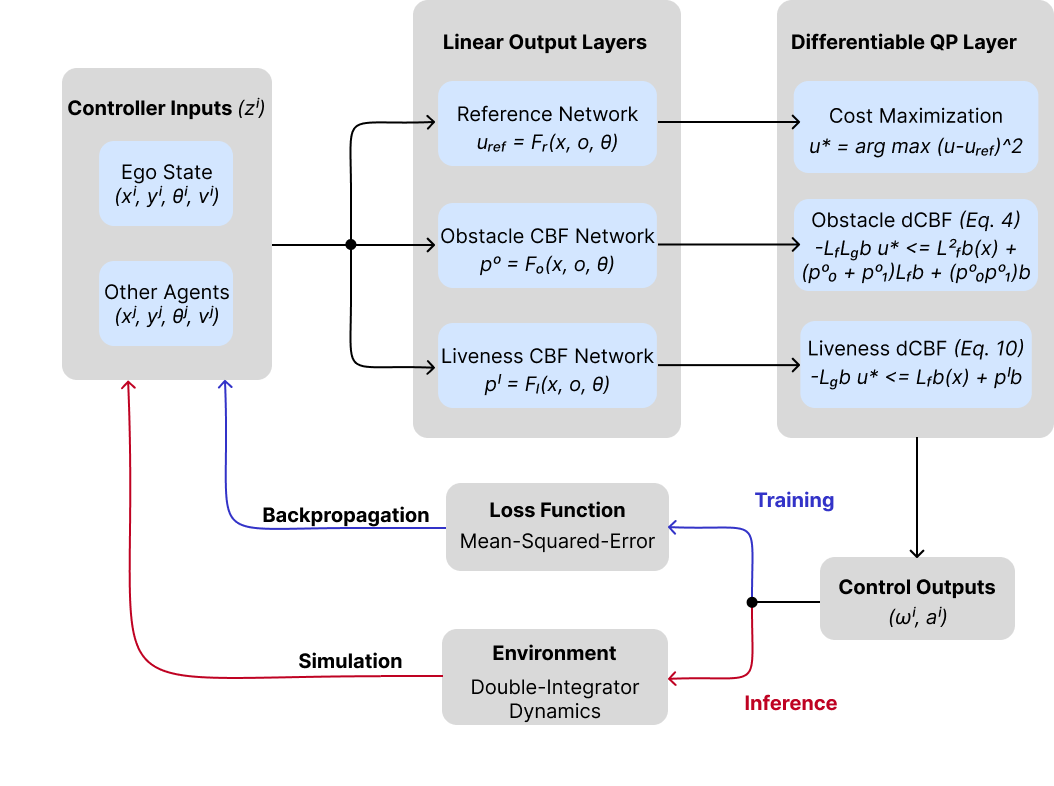}
\caption{\textbf{\modelname Architecture Overview} The ego state and observation inputs get fed into a feedforward network with three individual outputs: the reference control, the obstacle penalties, and the minimal invasiveness penalties. These three outputs get fed into a differentiable QP layer which solves a standard QP problem with inequality constraints (Equations \eqref{eqn:penalty_lie_ineq} and \eqref{eq: liveness_penalty_lie_ineq}) to enforce the CBFs. During backpropagation, the optimal reference value, as well as optimal penalty values for the CBF constraints, are learned.
}
  \label{fig: live_bnet_architecture}
\end{wrapfigure}

Each \modelname agent is defined by a neural network, where the function $\mathcal{F}: \mathcal{X} \times \Omega^i \rightarrow \mathcal{U}^i$ defines the feed-forward function of the model. The input to the network is the agent's state and observation, namely $z^i_t = (\mathbf{x}^i_t, o^i_t)$. The cost function for the model is defined via a mean-squared-error loss function such that $\mathcal{J}^i = \frac{1}{T} \sum_{t=0}^T (F(\widehat{z}_t^i) - \widehat{u}_t^i)^2$ for all corresponding $(\widehat{z}^i, \widehat{u}_t^i)$ in a trajectory dataset $(\widehat{\Gamma}^i, \widehat{\Psi}^i)$. The \modelname network (visualized in Figure~\ref{fig: live_bnet_architecture}) takes in the agent's state and observation as inputs, and passes it through three subnetworks, $F_r$, $F_o$, and $F_l$, which generate the reference control output $u_{ref}$, along with penalty values for the obstacle barrier function, $p^o$ and the liveness barrier function, $p^l$. $p^o$ determines the level of relaxation on the obstacle CBF (\ref{subsec: avoidance_dcbf}) and $p^l$ determines the level of relaxation on the liveness CBF (\ref{subsec: min_invasive_layer}). 
The network subsequently feeds these variables into a quadratic programming (QP) layer (\cite{amos2017optnet}) to maximize the function $(u - u_{ref})^2$ given differential CBF (dCBF) constraints defined by $p^o$ and $p^l$. During backpropagation, the mean squared error loss between the outputted $u$ and the optimal $\hat{u}$ (provided via the training dataset) is used to optimize the weights in the networks $F_r$, $F_o$, and $F_l$ via gradient descent.

The OptNet framework (\cite{amos2017optnet}) that \modelname is built on allows QP problems to be passed in with generalized inequality constraints in the form $G(z)u \leq h(z)$. Sections \ref{subsec: avoidance_dcbf} and \ref{subsec: min_invasive_layer} discuss the usage of Higher-Order CBFs (HOCBFs) in order to construct barrier function inequalities that are dependent on the control inputs. We refer the reader to [\cite{xiao2019control}] for more details. Our approach presents both a multi-agent differential CBF (dCBF) layer for multi-dimensional state spaces, as well as a liveness dCBF filter to ensure deadlock avoidance in a minimally invasive manner.

\subsection{Multi-Agent Collision Avoidance dCBFs}
\label{subsec: avoidance_dcbf}

We present an environment setup with double-integrator unicycle dynamics and a state space defined by $(x, y, \theta, v) \in \mathcal{X}$ where $x$ and $y$ represent the 2D position of the robot, $\theta$ represents the heading, and $v$ represents the forward velocity. The control inputs are defined by $(\omega, a) \in \mathcal{U}^i$ where $\omega$ represents turning velocity and $a$ represents linear acceleration. Observations of agent $i$ includes the positions, headings, and velocities of other, as well as the positions of static obstacles. We also refer to both agents and static obstacles under the general term \textit{obstacles} for this subsection, as agents are treated as moving obstacles whom we have no control over, where we utilize forwards dynamics to derive their future position based on their current velocity and heading. For any static obstacle $j$, we set $\theta^j = 0$ and $v^j = 0$. for static obstacles. Thus, from agent $i$'s perspective, the transition dynamics are defined as:


{\small \begin{equation}
\dot{\mathbf{x}} =
\begin{bmatrix}
\mathbf{f}^i(\mathbf{x}^i) \\ 
\mathbf{f}^j(\mathbf{x}^j)
\end{bmatrix}
+
\begin{bmatrix}
\mathbf{B} & \mathbf{0} \\ 
\mathbf{0} & \mathbf{0}
\end{bmatrix}
\begin{bmatrix}
\mathbf{u}^i \\ 
\mathbf{u}^j
\end{bmatrix},
\label{eq: multi-agent-state-space-system-equation}
\end{equation}}

where \(\dot{\mathbf{z}} = \begin{bmatrix} \dot{x^i}, \dot{y^i}, \dot{\theta^i}, \dot{v^i}, \dot{x^j}, \dot{y^j}, \dot{\theta^j}, \dot{v^j} \end{bmatrix}^\top\). The control inputs are applied using an input matrix, \(\mathbf{B}\), where 
\(\mathbf{B} = \begin{bmatrix} 0,0,1,0; 0,0,0,1\end{bmatrix}^\top\), 
and the control vectors for agents \(i\) and \(j\) are \(\mathbf{u}^i = \begin{bmatrix} \omega^i, a^i \end{bmatrix}^\top\) and \(\mathbf{u}^j = \begin{bmatrix} \omega^j, a^j \end{bmatrix}^\top\), respectively. The dynamics for each agent are represented by \(\mathbf{f}^i(\mathbf{x}^i)\) and \(\mathbf{f}^j(\mathbf{x}^j)\), where 
\(\mathbf{f}^i(\mathbf{x}^i) = \begin{bmatrix} v^i \cos(\theta^i), v^i \sin(\theta^i), 0, 0 \end{bmatrix}^\top\).
%
For simplicity, agents and static obstacles are considered to occupy circles of their respective radius $r$. The barrier function for obstacle avoidance can be written out as $b(z) = (x^i - x^j)^2 + (y^i - y^j)^2 - (r^i + r^j)^2 >= 0$ where $(x^j, y^j) \in \mathbb{R}^2$ represents the position of obstacle $j$, and $r^j$ represents the radius of obstacle $j$. As shown in (\cite{barriernet2021}), the HOCBF for $b(z)$ which has degree 2 with respect to the control outputs yields the following inequality:

{\small
\begin{equation}
    -L_fL_gb(z)u \leq L^2_fb(z) + (p^o_1(z) + p^o_2(z))L_fb(z) + (\dot{p^o_1}(z) + p^o_1(z)p^o_2(z))b(z)
    \label{eqn:penalty_lie_ineq}
\end{equation}
}

where $p^o_1(z)$ and $p^o_2(z)$ are trainable penalty functions, and $\dot{p^o_1}(z)$ is set to $0$. 
Given that $b(z) = (x^i - x^j)^2 + (y^i - y^j)^2 - (r^i + r^j)^2$, we can solve for the Lie derivatives of $b(z)$ in the $f(\mathbf{x})$ and $g(\mathbf{x})$ vector fields:

{\small
\begin{subequations}
    \begin{equation}
        L_fb(z) = 2(x^i - x^j) (v^i \cos(\theta^i) - {v^j} \cos(\theta^j)) + 2(y^i - y^j) (v^i \sin(\theta^i) - {v^j} \sin(\theta^j))
    \label{eqn: moving_obs_barrier}
    \end{equation}
    \begin{equation}        
        L_f^2b(z) = 2({v^i}^2 + {v^j}^2 - 2v^i{v^j}(\cos(\theta^i + \theta^j))
    \end{equation}
    \begin{equation}
        L_gL_fb(z) =
        \begin{bmatrix}
        -2(x^i - x^j)v^i \sin(\theta^i) + 2(y^i - y^j)v^i \cos(\theta^i) \\ 2(x^i - x^j) \cos(\theta^i) + 2(y^i - y^j) \sin(\theta^i)
        \end{bmatrix}^T
    \end{equation}
    \label{eqn:moving_obs_lie_deriv}
\end{subequations}
}

These values, plugged into Equation~\eqref{eqn:penalty_lie_ineq}, generate our differential CBF constraint.





\subsection{Minimally Invasive Deadlock Prevention dCBFs}
\label{subsec: min_invasive_layer}

We introduce a minimally invasive differential CBF layer to look ahead and output accelerations that avoid collisions and maintain liveness in SMG scenarios. We first check if agent $i$'s projected spatial path intersects with the projected spatial path of any other agent $j$ for all $j$ in $[1; k]$ s.t. $i \neq j$, assuming that agent $j$ maintains their heading ($\theta^j$) and velocity ($v^j$). This check boils down to a ray intersection, which is calculated as $a$ and $b$

{\small
\begin{equation}
    a = (\Delta y * \cos(\theta^i) - \Delta x * \sin(\theta^i)) * det , \qquad b = (\Delta y * \cos(\theta^j) - \Delta x * \sin(\theta^j)) * det
    \label{eq: check_intersection}
\end{equation}
}

where $det = \hat{{v^j}}_x * \hat{v^i}_y - \hat{{v^j}}_y * \hat{v^i}_x$ and $\hat{v^i} = \frac{v^i}{|v^i|}, \hat{{v^j}} = \frac{{v^j}}{|{v^j}|}$. The rays intersect if $a$ and $b$ are both greater than 0. The minimally invasive liveness filter is only applied if the rays intersect. Let $\widetilde{p}^i \in C^i$ be the closest point on agent $i$'s convex hull to agent $i$. Similarly, let $\widetilde{p}^j \in C^j$ be the closest point on agent $j$'s convex hull to agent $i$. We denote $\widetilde{p}^i$ and $\widetilde{p}^j$ to be the agents' \textit{critical points}. Assuming that the agents have not yet collided, these points lie on the boundaries of their respective agent's convex hull, and thus are $r^i$ and $r^j$ from the agents' current positions. Mathematically, 
{\small
\begin{equation}
    {p^i}^* = p^i + r^i * \hat{p}^{j-i}, \qquad {p^j}^* = p^j - r^j * \hat{p}^{j-i}
    \label{eq: collision_point}
\end{equation}
}
where $\hat{p}^{j-i} = \frac{p^j - p^i}{|p^j - p^i|}$. Thus, our problem reduces to avoiding a point-point collision instead of a convex-convex collision. We first calculate $c$, the potential collision point of $\widetilde{p}^i$ and $\widetilde{p}^j$, by projecting them forwards along the agents' current spatial path given $\theta$ and $v$:
\vspace{-10pt}
{\small
\begin{equation}
    \begin{split}
    \widetilde{p}^{i'} &= \widetilde{p}^i + \mathbf{v}^i, \qquad \qquad \quad \widetilde{p}^{j'} = \widetilde{p}^j + \mathbf{v}^j \\
    a^i &= \widetilde{x}^i * \widetilde{y}^{i'} - \widetilde{y}^i * \widetilde{x}^{i'}, \quad a^j = \widetilde{x}^j * \widetilde{y}^{j'} - \widetilde{y}^j * \widetilde{x}^{j'} \\
    c_x &= (a^j v^i \cos(\theta^i) - a^i v^j \cos(\theta^j)) / det, c_y = (a^j v^i \sin(\theta^i) - a^i v^j \sin(\theta^j)) / det \\
    \label{eq: collision_point}
    \end{split}
\end{equation}
}
where $det = v^i v^j (\cos(\theta^i)\sin(\theta^j) - \sin(\theta^i)\cos(\theta^j))$. We then calculate the distance from each agent's critical point $\widetilde{p}$ to the collision point.

{\small
\begin{equation}
    d^i = \sqrt{(\widetilde{x}^i - c_x)^2 + (\widetilde{y}^i - c_y)^2}, \qquad d^j = \sqrt{(\widetilde{x}^j - c_x)^2 + (\widetilde{y}^j - c_y)^2} 
\end{equation}
}

Given each agent's velocity, we calculate $t^i$ and $t^j$, the time for each agent's critical point, $\widetilde{p}$, to reach the collision point $c$ as $t = v / d$. We split the scenario into two different cases. If $t^i < t^j$, then that indicates that the ego agent, agent $i$, will pass the collision point before agent $j$. In this scenario, we enforce the barrier function $t^j > t^i \rightarrow b(z) = t^j - t^i = \frac{d^j}{{v^j}} - \frac{d^i}{v^i} \geq 0$. Since the $c$ lies along each agent's heading, $\theta$, and since a minimally invasive trajectory involves zero spatial deviation from the desired path, $v$ is the direct derivative of $d$. That is, $dd/dt = v$. Since our barrier function is with respect to $v^i$, and the control input $a^i$ appears in the first derivative of our barrier function, we alter the HOCBF inequality from Equation~\eqref{eqn:penalty_lie_ineq} to instead be
\begin{equation}
    -L_gb(z)u \leq L_fb(z) + p^lb(z)
    \label{eq: liveness_penalty_lie_ineq}
\end{equation}

Thus, we compute the Lie derivatives and formulate the CBF in Equation~\ref{eqn: faster_agent_live_CBF}.

{\small
\begin{equation}
    \begin{split}
    b(z) &= \delta (t^j - t^i), \qquad L_g b(z) = \delta\left(\frac{a^id^i}{{v^i}^2}\right), \qquad     -\delta\left(\frac{a^id^i}{{v^i}^2}\right) \leq p^l_{\delta}(x) \left(\frac{d^j}{{v^j}} - \frac{d^i}{v^i}\right)
    \end{split}
    \label{eqn: faster_agent_live_CBF}
\end{equation}
}

where $\delta(\cdot)$ is an indicator function and is equal to $1$ if the ego-agent is faster and $-1$ if it is slower and should yield to agent $j$. Note that there is no Lie derivative along the $f$ function since in the absence of any control outputs, the barrier function $t^j - t^i$ would remain constant over time. The penalty value, $p^l_\delta(z)$ is chosen for these CBFs to allow the network to learn the necessary constraint levels in each scenario.



\begin{theorem}
    Assuming $p^l_{-1}(z), p^l_1(z)$ are differentiable functions with respect to $z$, then the \modelname constraints in Equation~\eqref{eqn: faster_agent_live_CBF} guarantee the liveness of the system defined by~\eqref{eq: multi-agent-state-space-system-equation}.
    \label{thm: liveness}
\end{theorem}

\begin{proof}
    The proof follows directly from~(\cite{barriernet2021})(c.f. Theorem 2) and relies on the requirement that the relative degree of $p^l_{-1}(z), p^l_1(z)$ with respect to each component in $z$ is greater than or equal to that of the liveness constraints in Equations~\eqref{eqn: faster_agent_live_CBF}.

    In Equations~\eqref{eqn: faster_agent_live_CBF}, since the control input ($a^i$) appears in the first derivative, the relative degree is $1$. Next, recall that each agent $i$ has a partial observability over the positions and velocities of other robots in its neighborhood, that is, $z = [p^i, v^i, p^j, {v^j}]$. The penalty functions $p^l_{-1}(z), p^l_1(z)$ have a relative degree of $2$ with respect to position and $1$ with respect to velocity, therefore, the conditions set forth in~(\cite{barriernet2021}) are satisfied.
\end{proof}

\section{Experiments and Results}
\label{sec: experiments_and_results}

We aim to investigate two main questions: $(i)$ how does \modelname compare with existing multi-robot navigation methods in SMGs? and $(ii)$ how does \modelname compare with methods specifically designed to navigate SMGs?
To investigate these questions, we compare \modelname to five baseline methods. These baselines include two receding-horizon optimization-based controllers, \textsc{MPC-CBF}~(\cite{zeng2021safety}), which maintains safety from static and moving obstacles using CBFs and \textsc{SMG-CBF}~(\cite{chandra2024deadlock}), an extension of the \textsc{MPC-CBF} formulation that employs a threshold-based liveness CBF with steep acceleration outputs. Two multi-agent learning-based approaches were also tested: \textsc{MACBF}~(\cite{macbf2021}), a neural network controller that learns safety through separated action and CBF networks, and \textsc{PIC}~(\cite{liu2020pic}), which utilizes graph convolutional neural networks with a permutation invariant critic for scalable navigation. \modelname was additionally tested against BarrierNet~(\cite{barriernet2021}), which also employs a differentiable CBF layer, but lacks incentives to follow live behavior. We track standard navigation metrics, namely number of collisions and deadlocks, makespan (time to goal for the slower agent), and runtime per control iteration (in seconds) in a variety scenarios. To evaluate minimal invasiveness (or smoothness) between the different approaches, we measure average change in velocity and the average path deviation from the agent's desired trajectory. 
\begin{table}[t]
\centering
\resizebox{\linewidth}{!}{
\begin{tabular}{ rcccccc }
 \toprule[1.5pt]
 \multicolumn{7}{c}{\textit{Doorway Scenario}} \\
 \toprule
 Method & \# Collisions & \# Deadlocks & Makespan (s) & $\Delta V$ (m/s) & $\Delta$ Path (m) & Cycle Time (s) \\
 \midrule
 MPC-CBF~(\cite{zeng2021safety}) & $0$ & $50$ & $N/A$ & $0.003 \pm 0.000$ & $0.016 \pm 0.000$ & $90.9 \pm 0.9$ \\
 MACBF~(\cite{macbf2021}) & $50$ & $0$ & $N/A$ & $0.006 \pm 0.000$ & $0.149 \pm 0.048$ & $171.05 \pm 1.66$ \\
 PIC~(\cite{liu2020pic}) & $50$ & $0$ & $N/A$ & $0.031 \pm 0.006$ & $0.041 \pm 0.002$ & $0.3 \pm 0.0$ \\
 BarrierNet~(\cite{barriernet2021}) & $50$ & $0$ & $N/A$ & $0.004 \pm 0.000$ & $0.010 \pm 0.002$ & $7.3 \pm 0.0$ \\
 SMG-CBF~(\cite{chandra2024deadlock}) & $0$ & $0$ & $13.8 \pm 0.0$ & $0.009 \pm 0.000$ & $0.001 \pm 0.000$ & $81.1 \pm 0.3$ \\
 \midrule
 \modelname & $0$ & $0$ & $13.8 \pm 0.0$ & $0.002 \pm 0.000$ & $0.008 \pm 0.000$ & $7.5 \pm 0.0$ \\
\midrule
  \multicolumn{7}{c}{\textit{Intersection Scenario}} \\
 \midrule
 MPC-CBF~(\cite{zeng2021safety}) & $0$ & $50$ & $N/A$ & $0.006 \pm 0.000$ & $0.170 \pm 0.001$ & $302.6 \pm 3.9$  \\
 MACBF~(\cite{macbf2021}) & $50$ & $0$ & $N/A$ & $0.300 \pm 0.002$ & $0.009 \pm 0.004$ & $170.8 \pm 1.6$ \\
 PIC~(\cite{liu2020pic}) & $50$ & $0$ & $N/A$ & $0.081 \pm 0.024$ & $0.033 \pm 0.003$ & $0.3 \pm 0.0$ \\
 BarrierNet~(\cite{barriernet2021}) & $50$ & $0$ & $N/A$ & $0.008 \pm 0.000$ & $0.033 \pm 0.001$ & $7.3 \pm 0.0$ \\
SMG-CBF~(\cite{chandra2024deadlock}) & $0$ & $0$ & $12.2 \pm 0.0$ & $0.012 \pm 0.000$ & $0.000 \pm 0.000$ & $157.0 \pm 0.5$ \\
\midrule
\modelname & $0$ & $0$ & $11.6 \pm 0.0$ & $0.011 \pm 0.000$ & $0.000 \pm 0.000$ & $9.3 \pm 0.1$ \\
 \bottomrule[1.5pt]
\end{tabular}
}
\caption{Experiment results in the Doorway and Intersection scenarios, averaged over 50 runs.}
 \label{tab: doorway_scenario_results}
 \vspace{-20pt}
\end{table}


\subsection{Experiment Setup}
\label{sec: experiment_setup}
The simulation environment was setup in Python using the \texttt{do\_mpc} framework package (\cite{LUCIA201751}), which is based on Casadi (\cite{andersson2019casadi}). The kindoynamic constraints remained constant across all scenarios, with a maximum velocity of $v = 0.3~m/s$, a maximum acceleration / deceleration of $a = 0.1~m/s^2$, and a maximum angular velocity of $\omega = 0.5~rad/s$. Agent radii of $0.1~m$, simulation step time of $0.2~s/iteration$ and simulation time of $18s$ were also maintained constant across all experiments. The following symmetric SMGs were tested: $(i)$ \textit{Doorway scenario}: multiple robots pass through a doorway with width $0.3m$, only large enough for one robot to fit at a time. The agents started at maximum velocity ($0.3m$), $2m$ to the left of the doorway and $0.5m$ north / south of it. \textit{Intersection scenario}: multiple robots cross an intersection that is $0.35m \times 0.35m$, allowing only one robot to pass at a time. The robots started $1m$ away from the intersection with their goal placed $1m$ past the intersection.



\paragraph{Training:} To train \modelname, we generated data using an optimal, receding-horizon MPC controller (\cite{zeng2021safety}) across various perturbations of the doorway and intersection scenarios. We optimized the MPC agent to create minimally invasive solutions that avoid deadlocks and collisions using the barrier functions defined in Equations~\ref{eqn: moving_obs_barrier} and \ref{eqn: faster_agent_live_CBF}. For each scenario, defining parameters such as starting state, goal state, and gap size were slightly perturbed to generate a data suite containing variations of the original scenario, thereby increasing the diversity and quantity of training data. For each scenario perturbation, the MPC's state and input cost matrices, as well as the CBF parameters (\cite{zeng2021safety}), were individually tuned to generate an optimal trajectory for that scenario. \modelname was subsequently trained on this augmented data through offline supervised learning, using mean squared error as the loss function. The \modelname network consists of single, linear layer with 256 nodes followed by three parallel linear layers, $F_r, F_o, F_l$, each with 64 nodes. We use ReLU as the activation function between layers. The training process spanned 30 epochs of shuffled data, with a batch size of 64 and a learning rate of 0.001. 

Baseline learning methods (Figure \ref{tab: doorway_scenario_results}) were trained within their native training loops and environments. These methods were subsequently adapted to the SMG environment described in Section~\ref{sec: experiment_setup} through custom state-action mapping interfaces to ensure environment consistency.

\subsection{Results}

For each agent, $50$ Doorway and Intersection scenarios were run to test the safety, liveness, and smoothness of the trajectories of each agent. The accumulated and averaged metric values are displayed in Table \ref{tab: doorway_scenario_results} and the resulting trajectories are shown in Figure \ref{fig: result_trajectories}. The \textsc{MACBF}, \textsc{PIC}, and BarrierNet models resulted in collisions due to mere soft constraints on safety with the limited training data. The \textsc{MPC-CBF} model was able to avoid collisions, but succumbed to deadlocks due to an inability to make safe progress. Additionally, \modelname performed minimally invasive behavior as it maintained its desired spatial trajectory and minimally perturbed its velocity, approaching the CBF threshold without ever violating it (Figure \ref{fig: cbf_violations_and_desired_traj}).

\begin{figure}[t]
\subfigure[{\footnotesize MPC-CBF}]{%
    \label{fig: mpc_cbf_doorway_scenario}
    \includegraphics[width=0.15\linewidth]{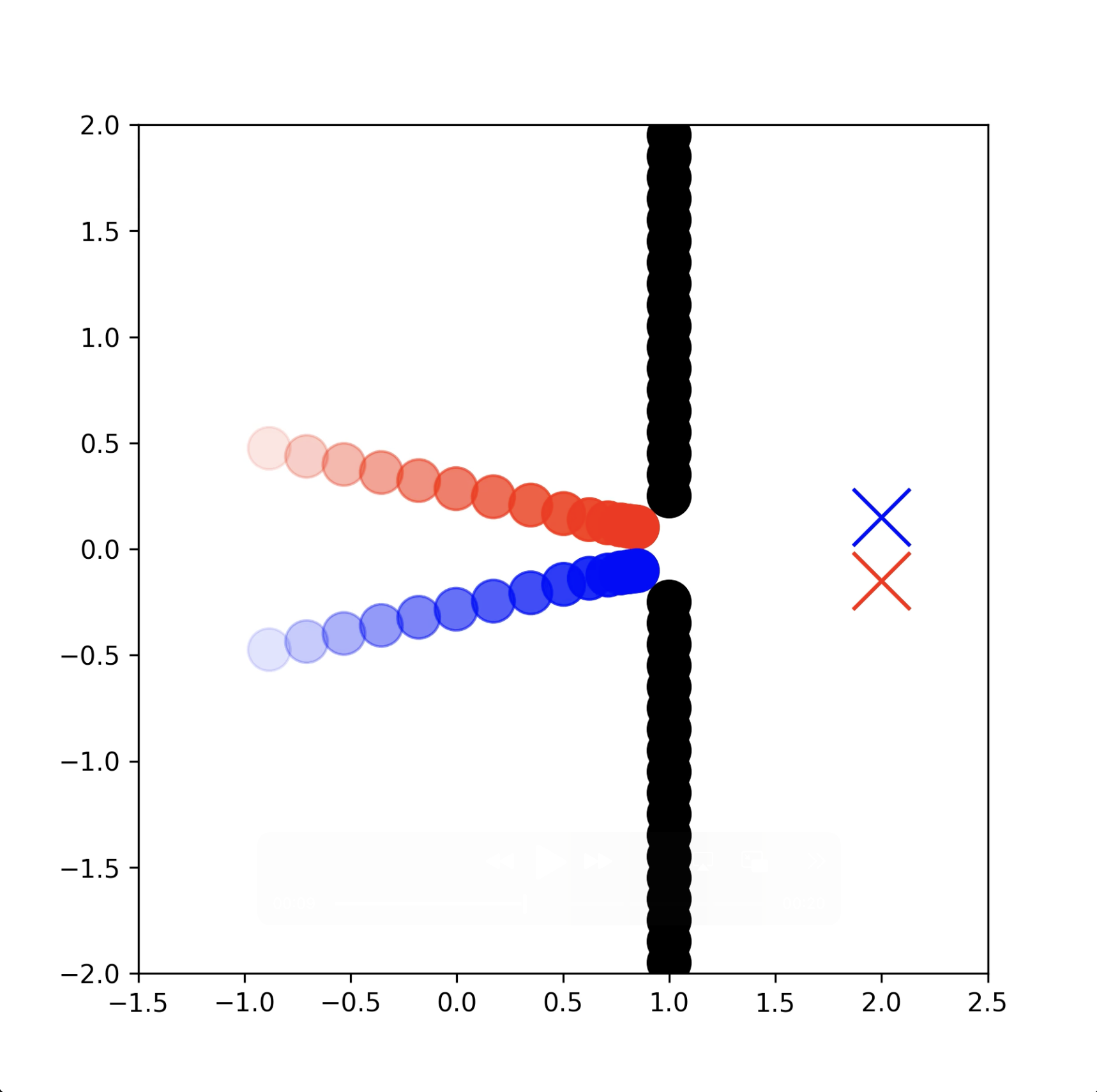}
}
\subfigure[{\footnotesize SMG-CBF}]{%
    \label{fig: smg_cbf_doorway_scenario}
    \includegraphics[width=0.15\linewidth]{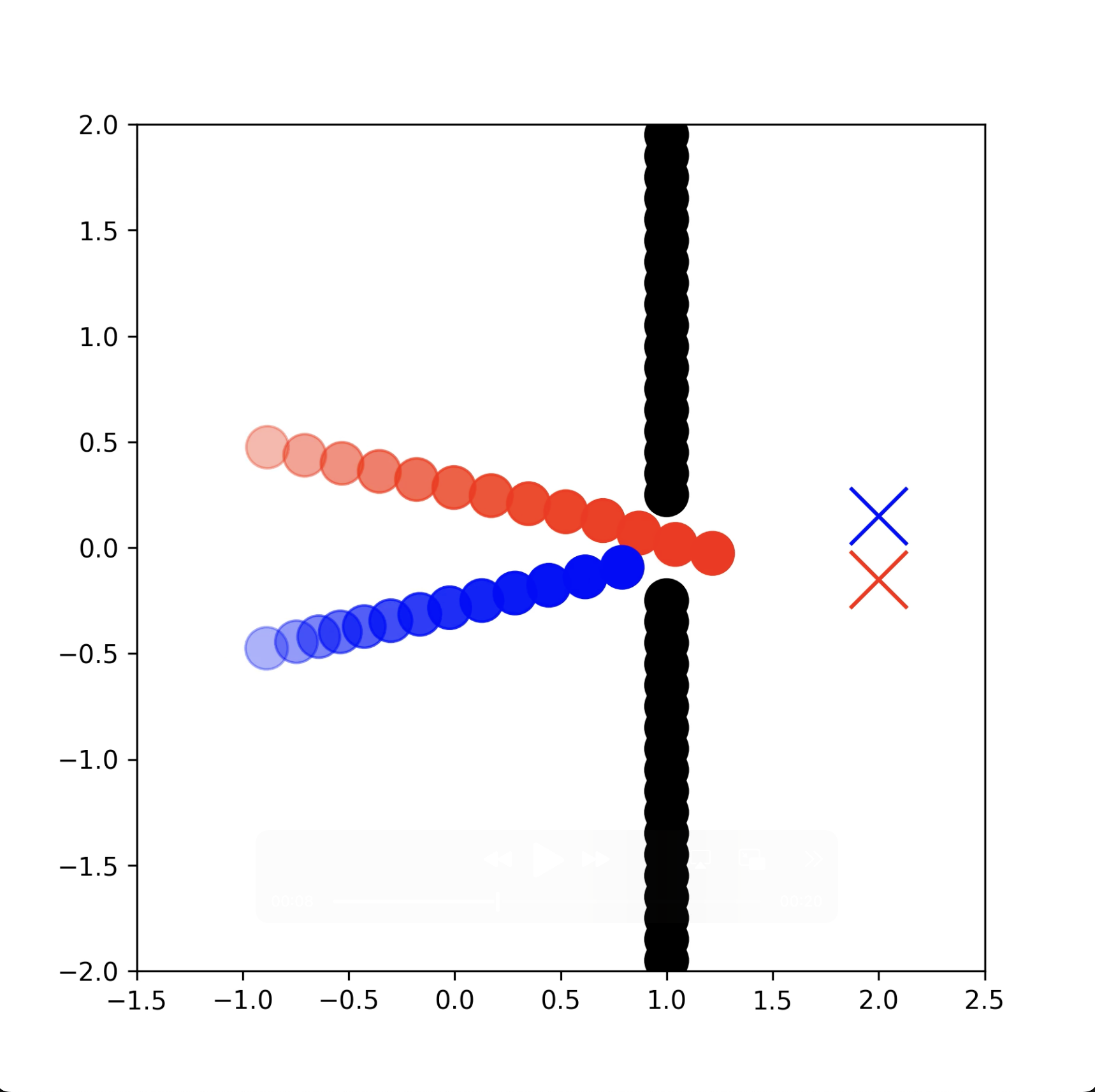}
}
\subfigure[{\footnotesize MACBF}]{%
    \label{fig: macbf_doorway_scenario}
    \includegraphics[width=0.15\linewidth]{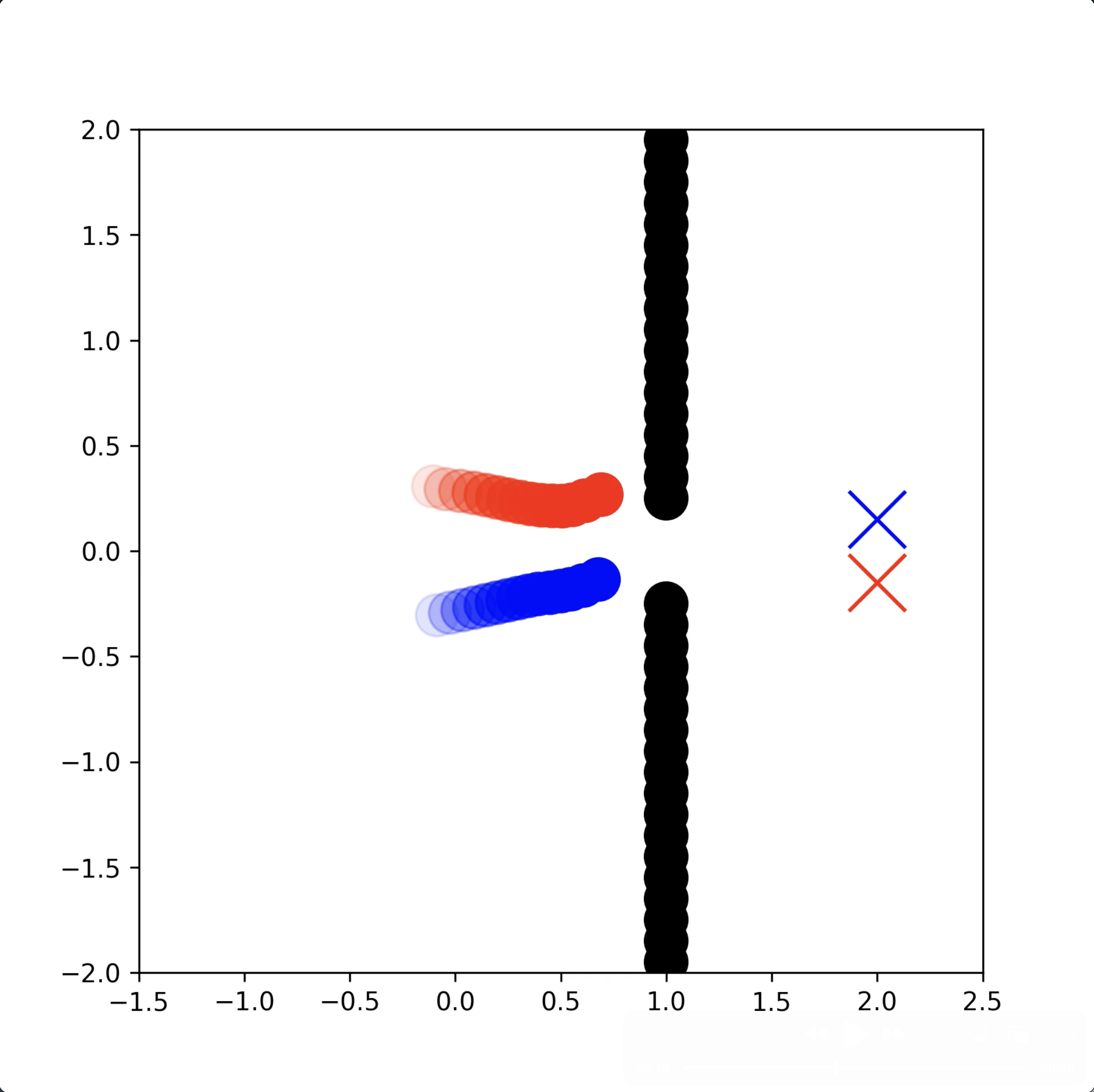}
}
\subfigure[{\footnotesize PIC}]{%
    \label{fig: pic_doorway_scenario}
    \includegraphics[width=0.15\linewidth]{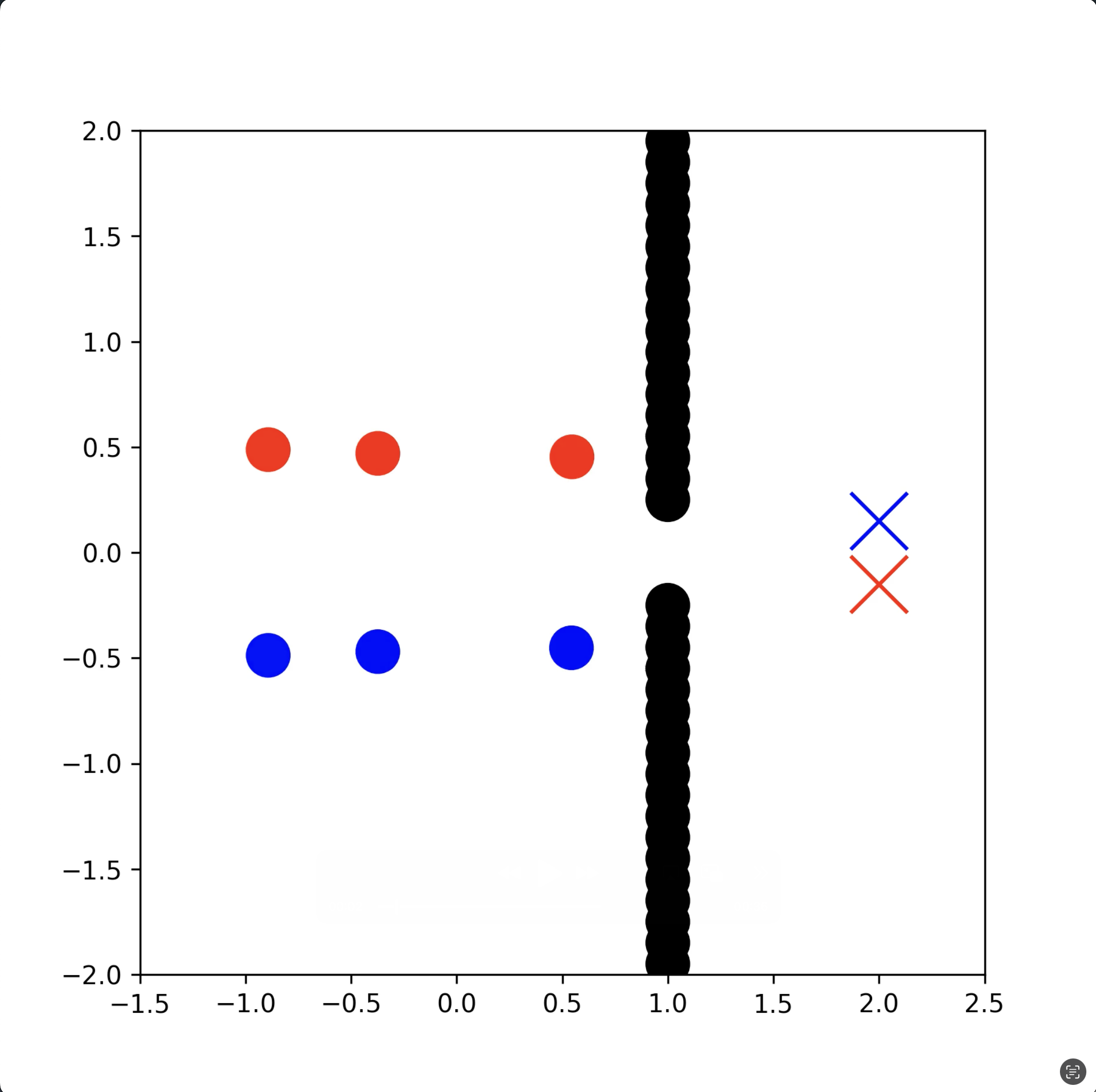}
}
\subfigure[{\footnotesize BarrierNet}]{%
    \label{fig: barriernet_doorway_scenario}
    \includegraphics[width=0.15\linewidth]{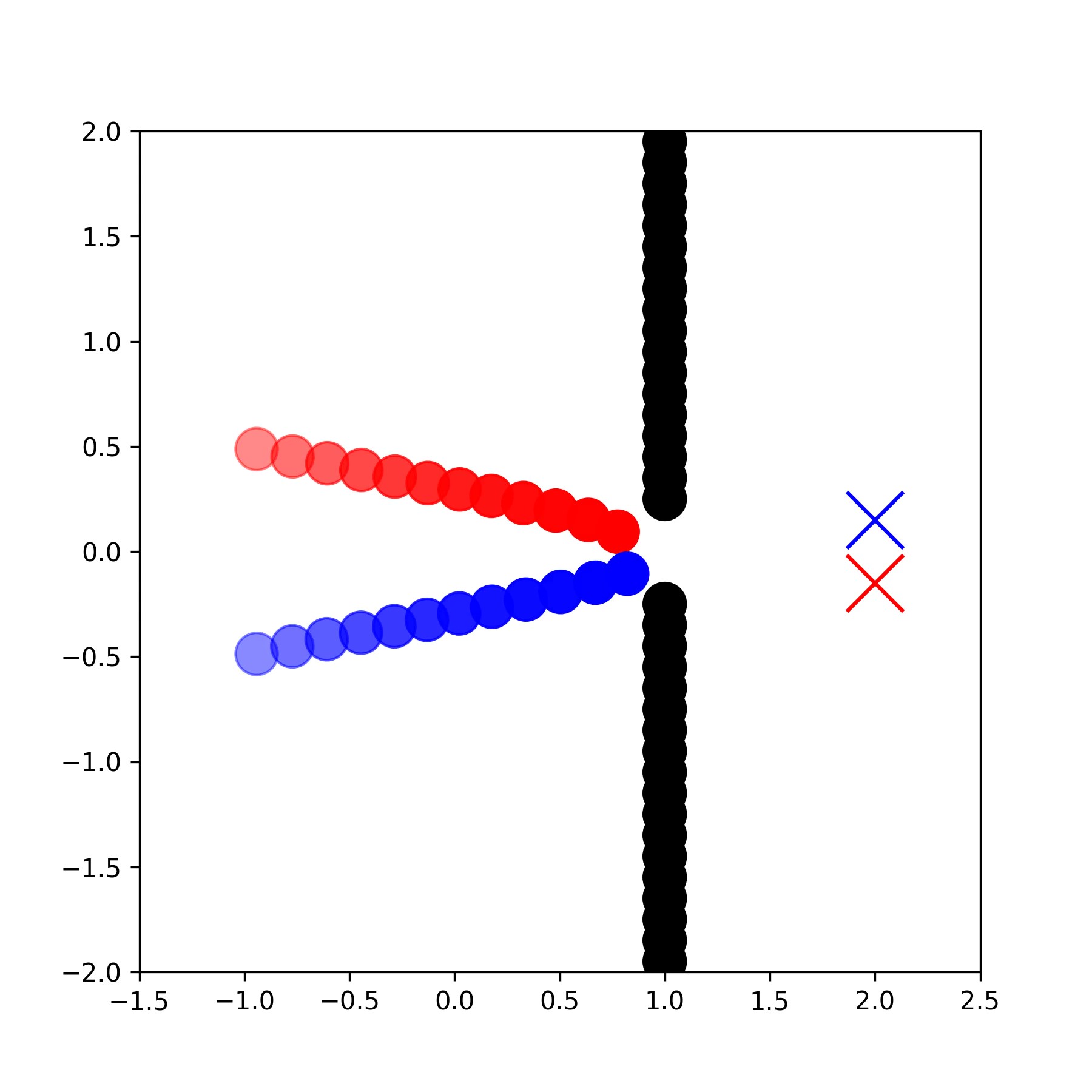}
}
\subfigure[{\footnotesize \modelname}]{%
    \label{fig: livenet_doorway_scenario}
    \includegraphics[width=0.15\linewidth]{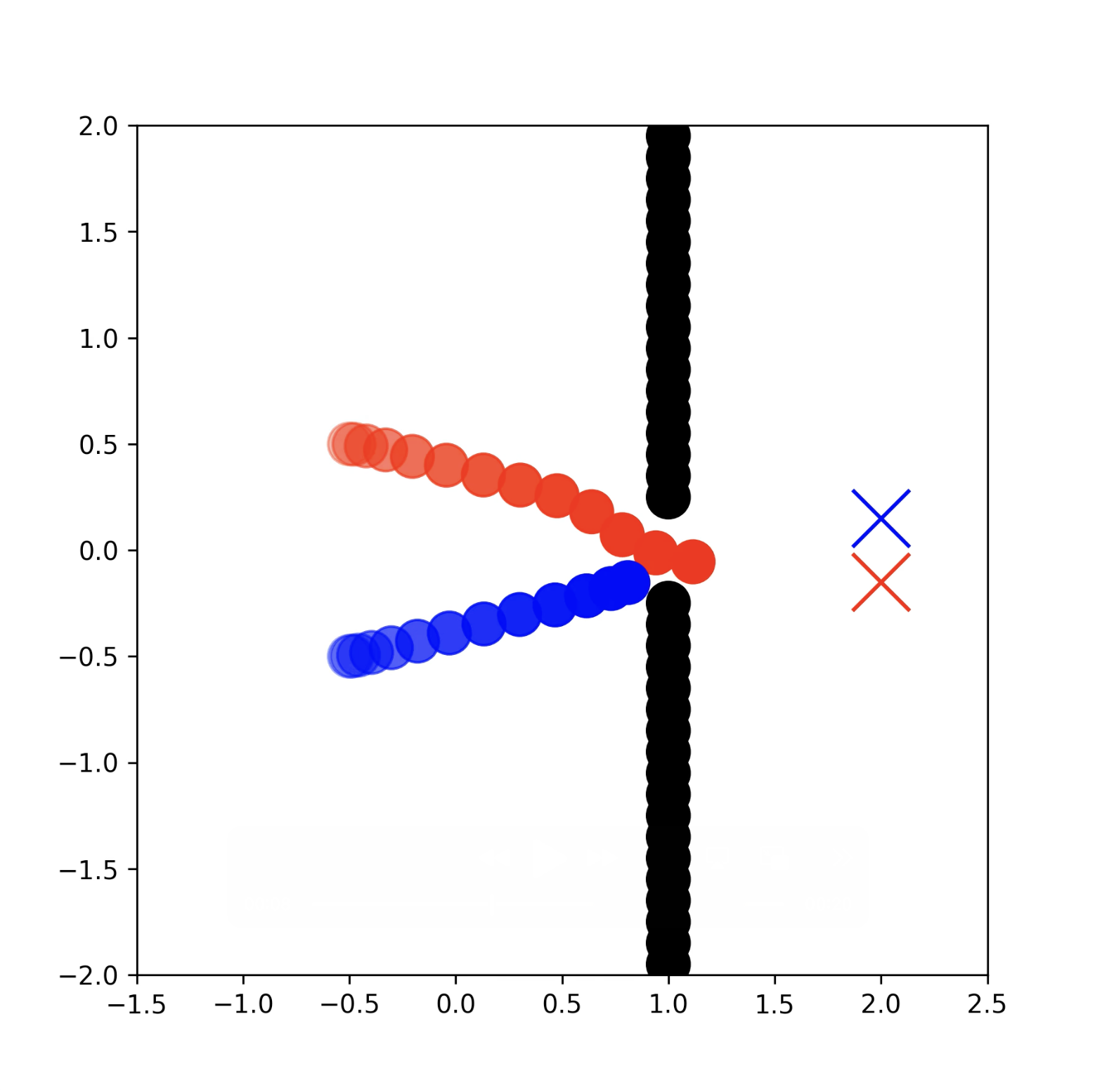}
}

\subfigure[{\footnotesize MPC-CBF}]{%
    \label{fig: mpc_cbf_intersection_scenario}
    \includegraphics[width=0.15\linewidth]{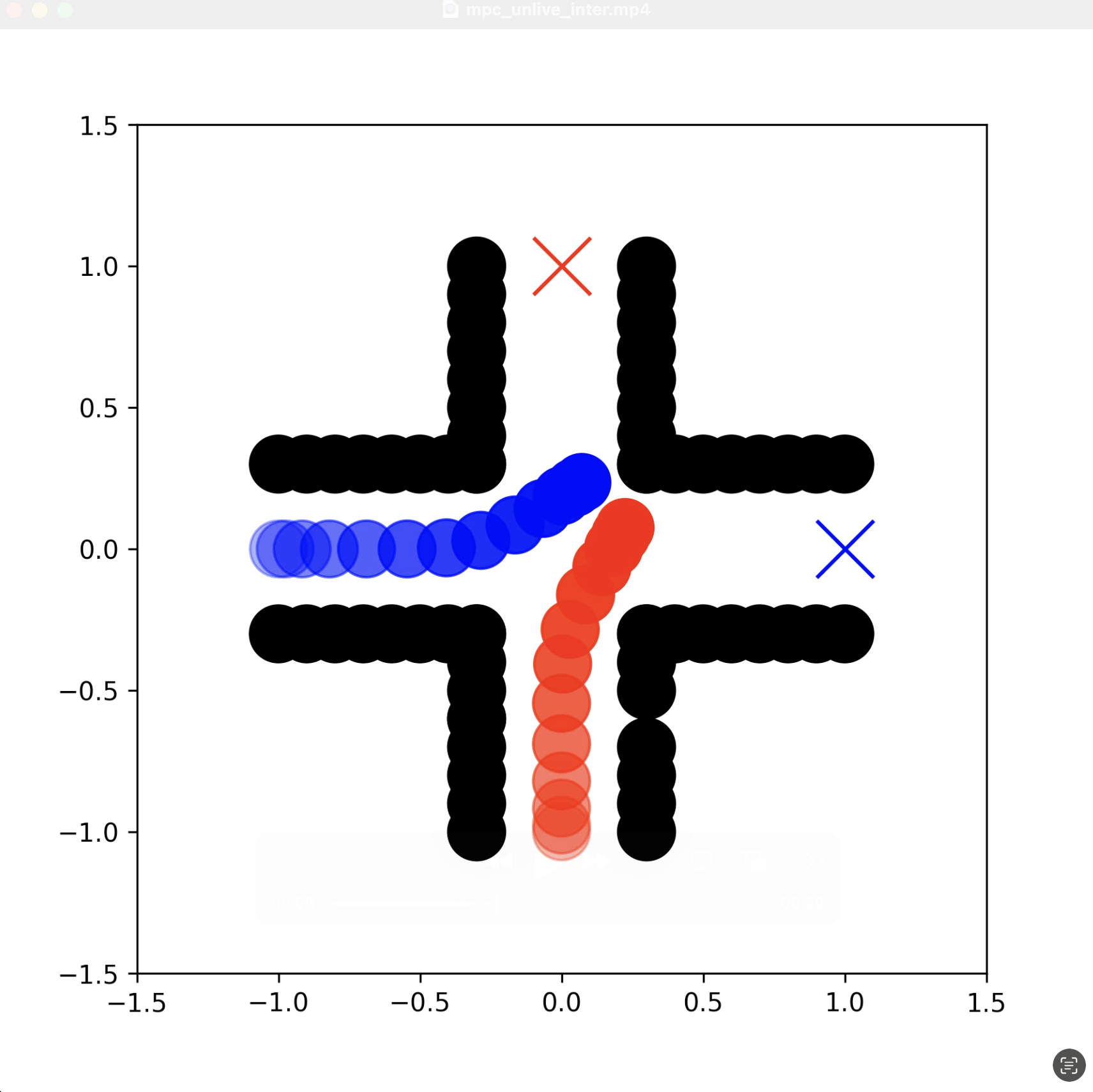}
}
\subfigure[{\footnotesize SMG-CBF}]{%
    \label{fig: smg_cbf_intersection_scenario}
    \includegraphics[width=0.15\linewidth]{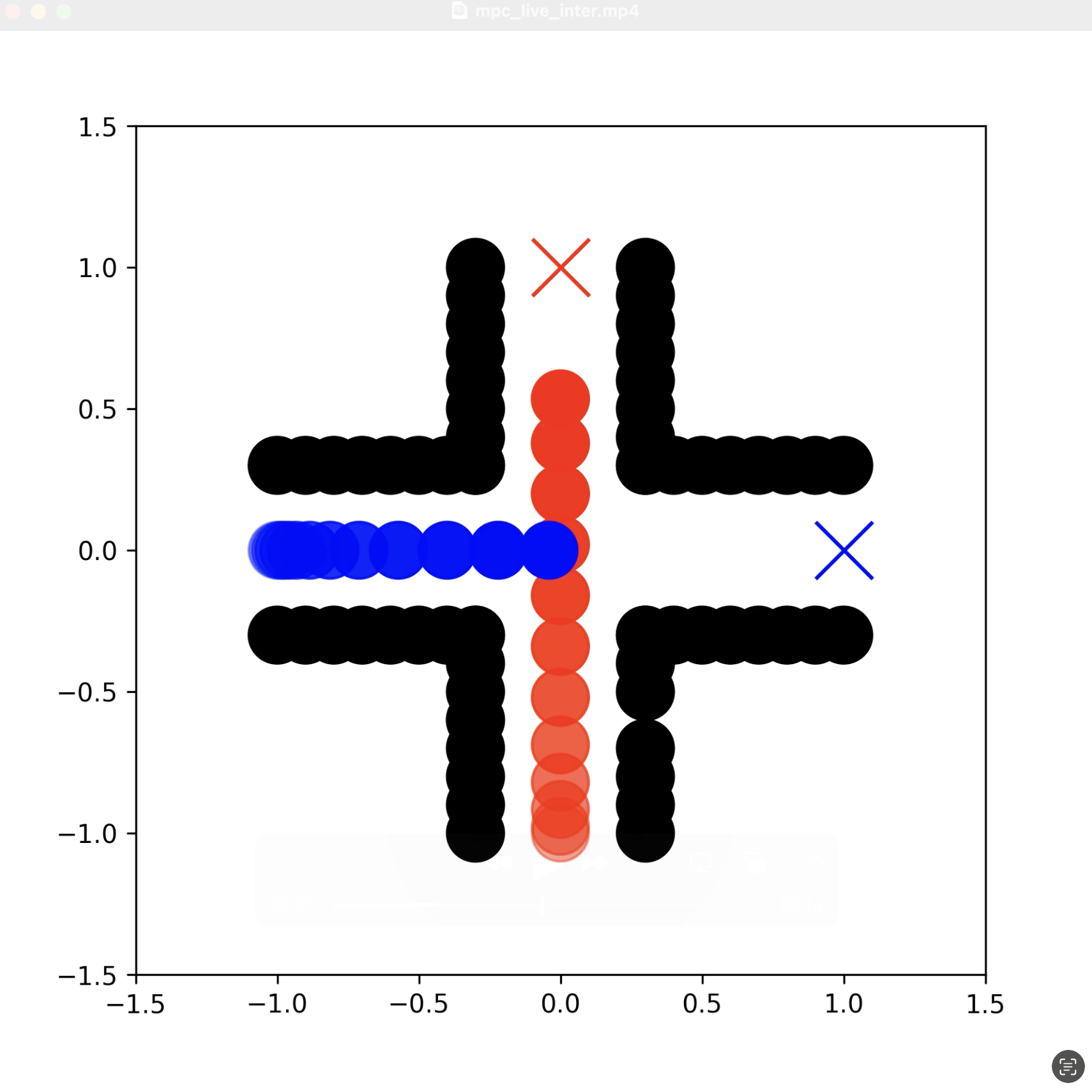}
}
\subfigure[{\footnotesize MACBF}]{%
    \label{fig: macbf_intersection_scenario}
    \includegraphics[width=0.15\linewidth]{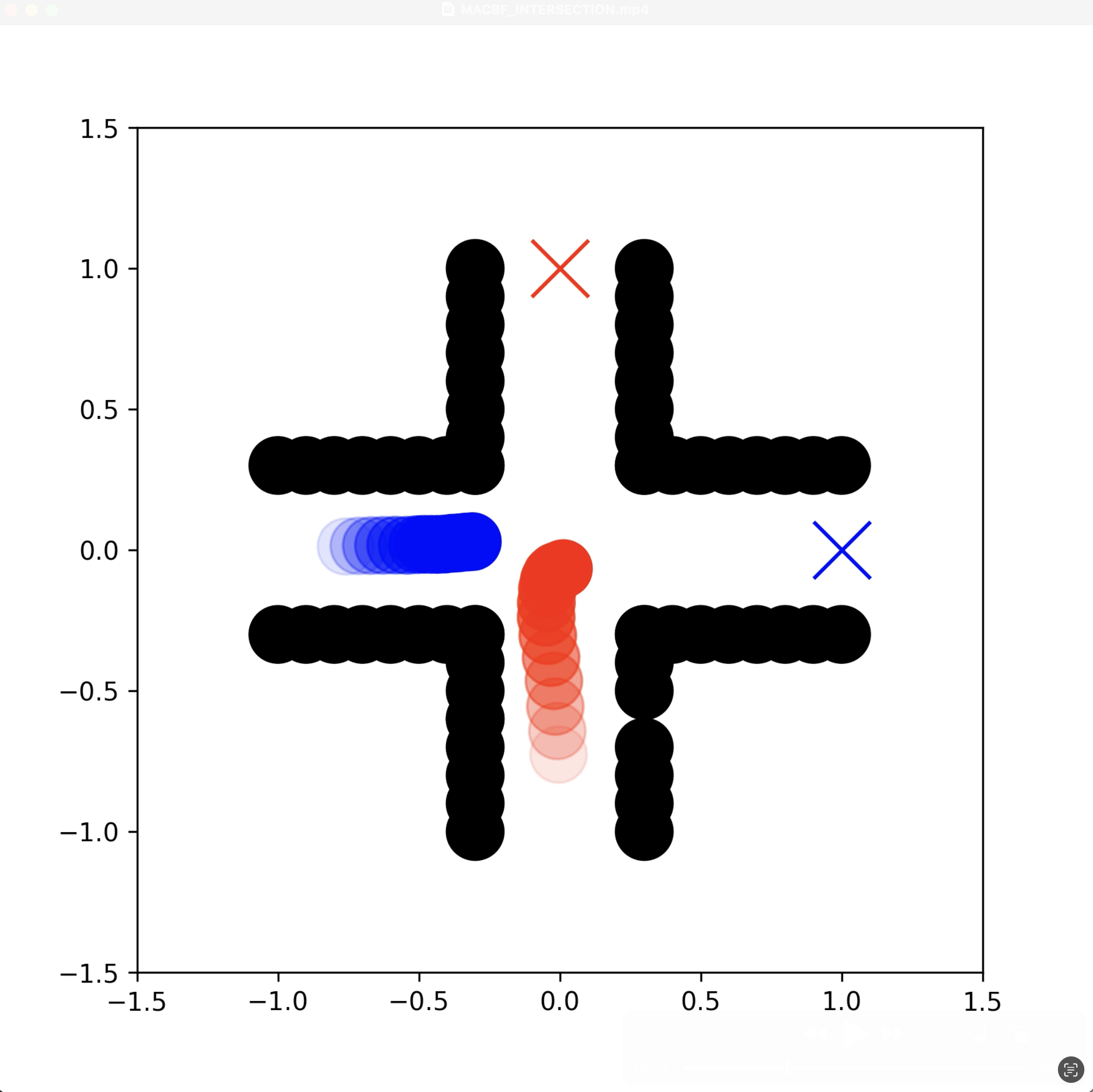}
}
\subfigure[{\footnotesize PIC}]{%
    \label{fig: pic_intersection_scenario}
    \includegraphics[width=0.15\linewidth]{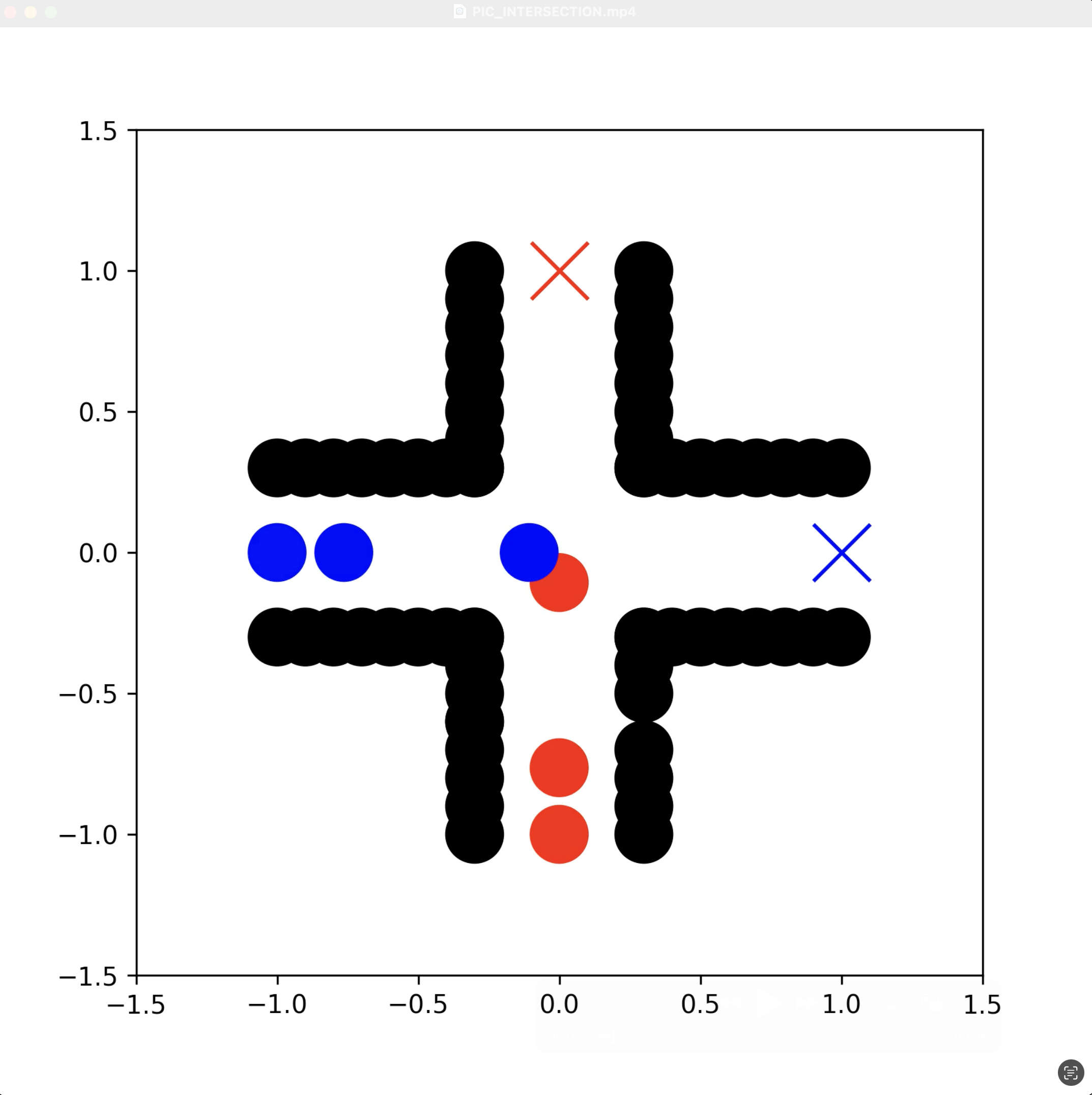}
}
\subfigure[{\footnotesize BarrierNet}]{%
    \label{fig: barriernet_intersection_scenario}
    \includegraphics[width=0.15\linewidth]{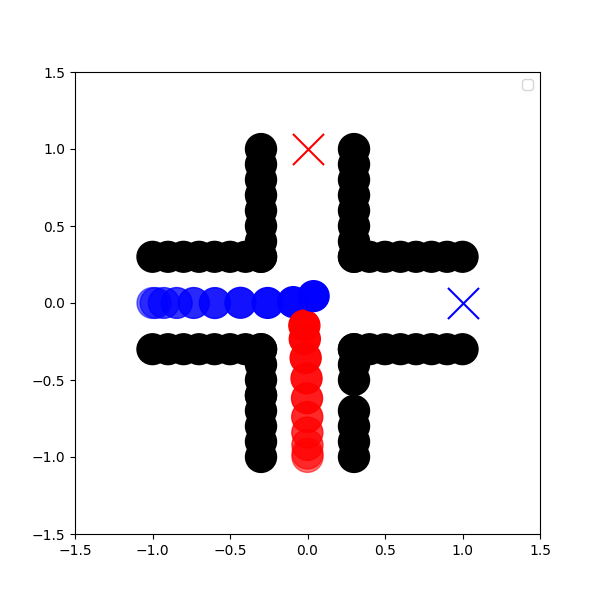}
}
\subfigure[{\footnotesize \modelname}]{%
    \label{fig: livenet_intersection_scenario}
    \includegraphics[width=0.15\linewidth]{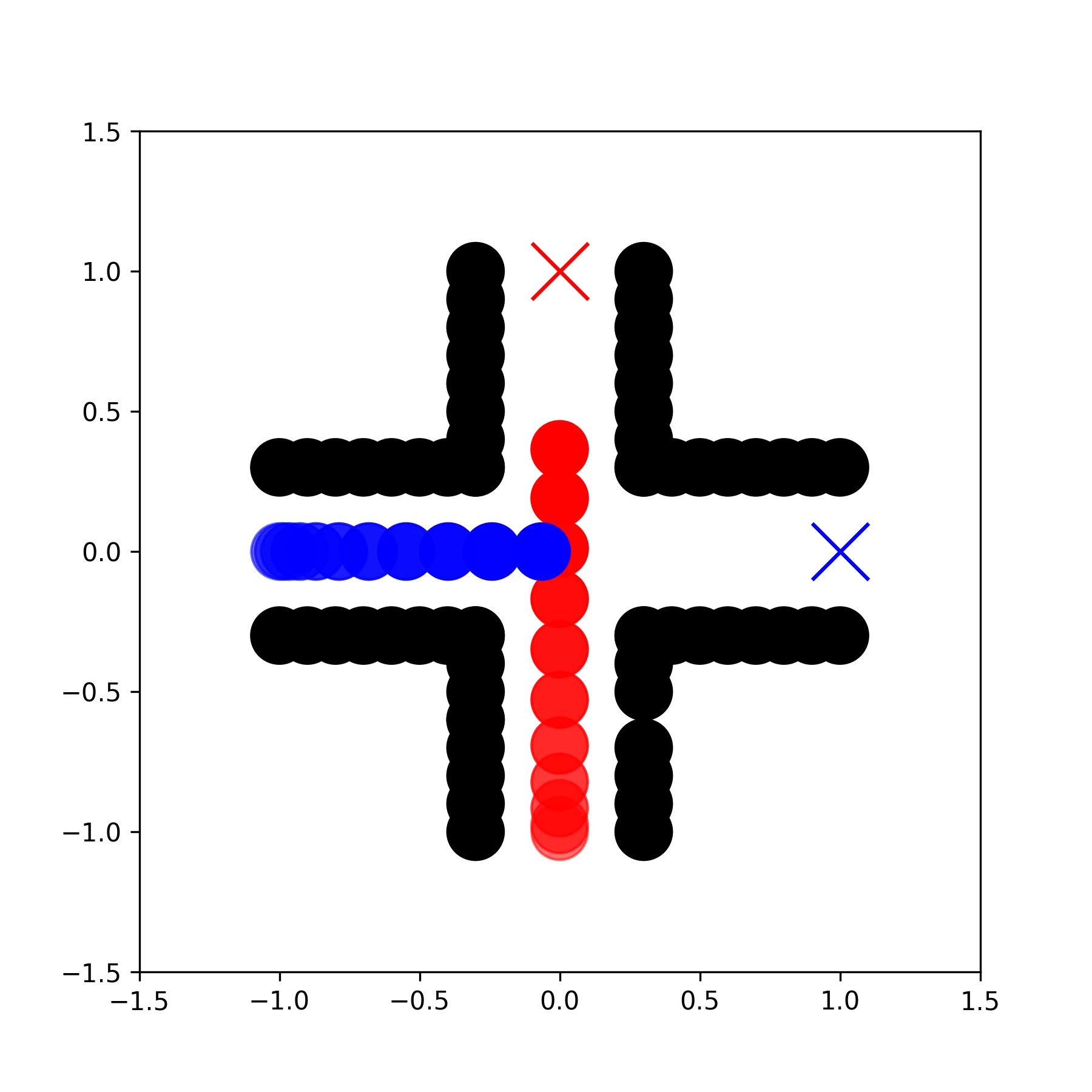}
}
\caption{Resulting trajectories in Doorway (Figures~\ref{fig: mpc_cbf_doorway_scenario}-\ref{fig: livenet_doorway_scenario}) and Intersection (Figures~\ref{fig: mpc_cbf_intersection_scenario}-\ref{fig: livenet_intersection_scenario})}
\label{fig: result_trajectories}
\end{figure}

\begin{figure}[t]
\subfigure[{\footnotesize Distance CBF $\left(p^o\right)$}]{%
    \label{fig: livenet_doorway_vel}
    \includegraphics[width=0.32\linewidth]{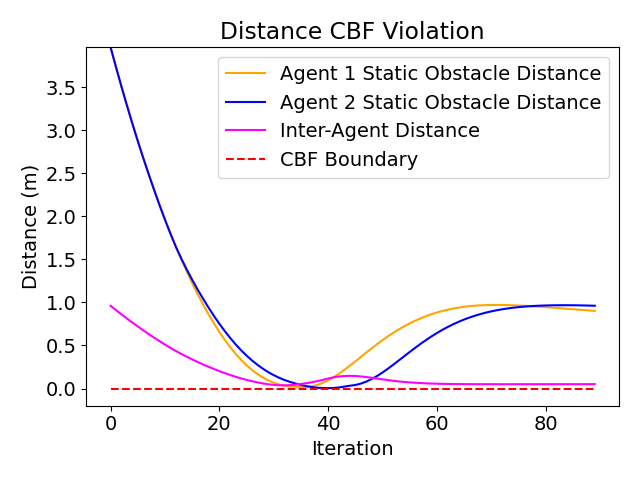}
}%
\subfigure[{\footnotesize Liveness CBF $\left(p^l\right)$}]{%
    \label{fig: livenet_doorway_vel}
    \includegraphics[width=0.32\linewidth]{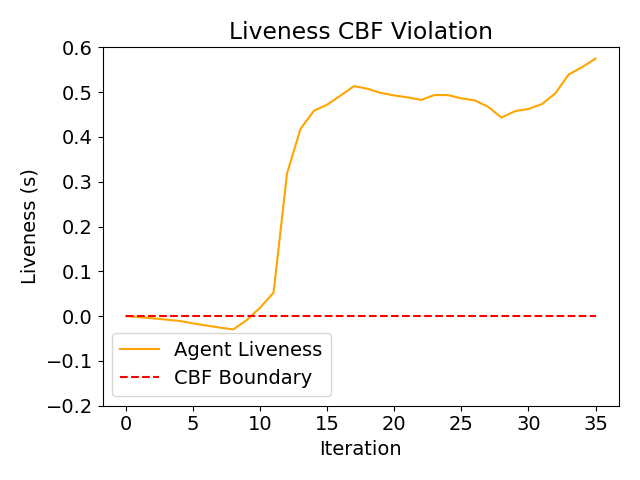}
}%
\subfigure[{\footnotesize Average $\Delta$ Path}]{%
    \label{fig: livenet_doorway_vel}
    \includegraphics[width=0.32\linewidth]{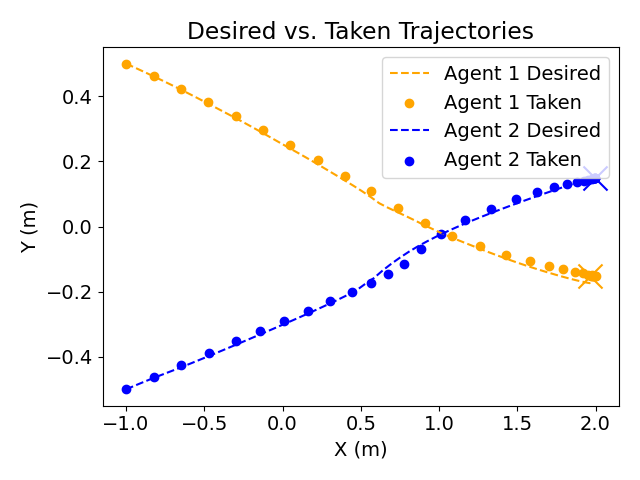}
}

\caption{\modelname's obstacle dCBF, liveness dCBF, and deviation from desired path in Doorway scenario.}
\label{fig: cbf_violations_and_desired_traj}
\end{figure}

Of the baselines, only \textsc{SMG-CBF} was successful due to its deadlock resolution capabilities. \textsc{SMG-CBF}, however, is $10\times$ slower than \modelname as it runs an iterative optimization algorithm every control cycle. Additionally, as shown in the Doorway scenario results, \textsc{SMG-CBF} was significantly more invasive, with an average velocity perturbation of $4$--$5\times$ that of \modelname. Another core limitation with \textsc{SMG-CBF} is its dependence on a \textit{liveness threshold} to determine when to apply the CBF (\cite{chandra2024deadlock}). As the controller enters and exits this threshold, its acceleration output varies significantly, thus producing jagged controls as shown in Figure~\ref{fig: smg_cbf_doorway_vel} between iterations $10$ and $40$. On the other hand, \modelname's ability to learn how much to relax the CBF as a function of the agent's state and observation allows the velocity profile of the resultant path to be much smoother as demonstrated by the smoother dip in agent $2$'s velocity in Figure~\ref{fig: livenet_doorway_vel}, which more closely mimics human-like behavior. Figure~\ref{fig: vel_profile_comparison} further shows evidence of better human-like yielding for \modelname. In particular, note that agent $2$ does not begin to yield until the last second (around iteration $20$) compared to SMG-CBF where agent $2$ begins to slow around iteration $10$, suggesting that \modelname results in less conservative, more agile navigation.
\begin{wrapfigure}{r}{0.6\linewidth}
\vspace{-10pt}
\centering
\subfigure[{\footnotesize \modelname}]{%
    \label{fig: livenet_doorway_vel}
    \includegraphics[width=0.47\linewidth]{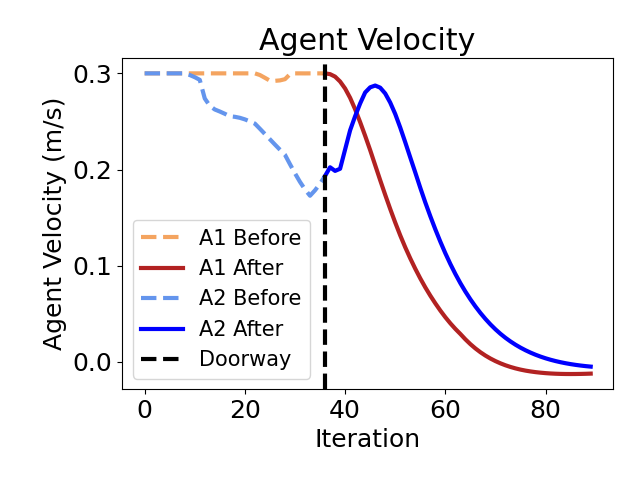}
} 
\subfigure[{\footnotesize SMG-CBF}]{%
    \label{fig: smg_cbf_doorway_vel}
    \includegraphics[width=0.47\linewidth]{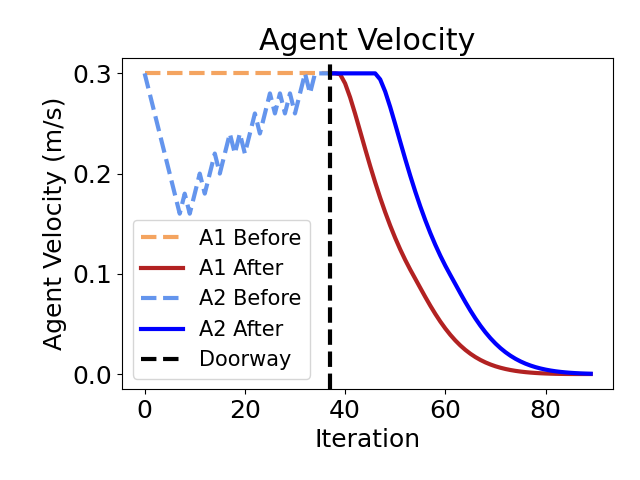}
} 
\caption{Comparing agents' (A1, A2) velocities generated by \modelname and \textsc{SMG-CBF}, before and after crossing the doorway.}
\label{fig: vel_profile_comparison} 
\vspace{-10pt}
\end{wrapfigure}

\modelname is also more robust to variations in the environment and agent configurations than\textsc{SMG-CBF}. We tested both on a suite of $28$ perturbed scenarios of the original doorway SMG without any changes to their parameter configuration. The perturbations were created by variations in the agents' initial position, initial heading, initial velocity, and goal position. The positions were altered on a scale of $0.5m$, initial headings facing the doorway and facing the wall were tested, and the initial velocity was either full speed ($0.3m/s$) or standstill ($0.0m/s$). \modelname was able to solve {25 / 28} scenarios without a deadlock or collision, whereas \textsc{SMG-CBF} was only able to solve {16 / 28}.

It should also be noted that due to \textsc{SMG-CBF}'s deterministic manner, it is unable to solve perfectly symmetrical cases without predefining which agent should start off moving faster. On the other hand, the same \modelname network could be used for both agents, as it has a slight inherent bias based on its starting and goal positions, allowing it to break the symmetry. \textsc{SMG-CBF}'s lack of robustness is due to constant parameters defining how strictly the CBF is followed for each scenario. On the other hand, \modelname's ability to predict the penalty values, $p(z)$, that define the relaxation of the CBF allows it to better adapt to a multitude of scenarios.

\section{Conclusion}
\label{sec: conclusion}

In this work, we presented \modelname, a robust, minimally-invasive neural network controller that uses differentiable CBF layers to tackle the safety and liveness challenges of constrained environments. Our navigation approach utilized the BarrierNet framework as the base neural network. We introduced novel differentiable CBF layers to provide liveness and 2D multi-agent navigation. To train the network, we hand-tuned an optimal recending-horizon controller over many perturbed scenarios to generate a large dataset. Our approach guarantees safe and live behavior given enough training data and iterations to learn the dCBF's corresponding penalty values. Experiments show that in practical scenarios the model outperforms existing solutions in safety, minimal invasiveness, compute speed, and robustness. The faster compute time and robustness are crucial when run on real robots in constrained areas that need to react quickly to unpredictable situations.

Our approach has some limitations. \modelname is currently tested in simulation, and we plan to deploy in the real world in the future. Additionally, \modelname has only been tested on 2-agent scenarios. We plan on investigating the scalability of this model in terms of the compute time and accuracy as the number of agents increases. Furthermore, an issue with imitation learning is a laborious data generation process, as we have to tune the optimal controller for each scenario perturbation. Utilizing unsupervised learning methods would allow for self-exploration of novel states instead of forcing the agent to only learn from states that the optimal controller explored.


\bibliography{refs}

\end{document}